\documentclass{article}

\usepackage[nonatbib, preprint]{neurips_2021}

\usepackage[utf8]{inputenc} % allow utf-8 input
\usepackage[T1]{fontenc}    % use 8-bit T1 fonts
\usepackage{hyperref}       % hyperlinks
\usepackage{url}            % simple URL typesetting
\usepackage{booktabs}       % professional-quality tables
\usepackage{amsfonts}       % blackboard math symbols
\usepackage{nicefrac}       % compact symbols for 1/2, etc.
\usepackage{microtype}      % microtypography
\usepackage{mathtools}
\usepackage{amsmath}
\usepackage{amssymb}
\usepackage{amsthm}
\usepackage{tabularx}
\usepackage{color}
\usepackage{xcolor} 
\usepackage{multirow}
\usepackage{algorithm}
\usepackage{algorithmic}
\usepackage{placeins}

\usepackage{mathrsfs}
\usepackage{ulem}
\newtheorem{proposition}{Proposition}
\newtheorem{theorem}{Theorem}
\newtheorem{innercustomthm}{Theorem}
\newenvironment{customthm}[1]
  {\renewcommand\theinnercustomthm{#1}\innercustomthm}
  {\endinnercustomthm}

\newtheorem{innerlemma}{Lemma}
\newenvironment{customlemma}[1]
  {\renewcommand\theinnerlemma{#1}\innerlemma}
  {\endinnerlemma}

\newtheorem{lemma}{Lemma}
\newtheorem{remark}{Remark}

\usepackage{listings}
\definecolor{codegreen}{rgb}{0,0.6,0}
\definecolor{codegray}{rgb}{0.5,0.5,0.5}
\definecolor{codepurple}{rgb}{0.58,0,0.82}
\definecolor{backcolour}{rgb}{0.95,0.95,0.92}
\title{Reliable Estimation of KL Divergence using a Discriminator in Reproducing Kernel Hilbert Space  }
\author{%
  Sandesh Ghimire, Aria Masoomi, Jennifer Dy  \\
  Department of Electrical and Computer Engineering\\
  Northeastern University\\
  Boston, MA, USA\\
  \texttt{sandesh@ece.neu.edu, a.masoomi@northeastern.edu,  jdy@ece.neu.edu} \\
}

\begin{document}

\maketitle
\vspace{-0.2cm}
\begin{abstract}
% Several scalable sample-based methods to compute the Kullback–Leibler (KL) divergence between two distributions have been proposed and applied in large-scale machine learning models. While they have been found to be unstable, the theoretical root cause of the problem is not clear.
Estimating Kullback–Leibler (KL) divergence from samples of two distributions is essential in many machine learning problems. Variational methods using neural network discriminator have been proposed to achieve this task in a scalable manner. However, we noted that most of these methods using neural network discriminators suffer from high fluctuations (variance) in estimates and instability in training. In this paper, we look at this issue from statistical learning theory and function space complexity perspective to understand why this happens and how to solve it. We argue that the cause of these pathologies is lack of control over the complexity of the neural network discriminator function and could be mitigated by controlling it. To achieve this objective, we 1) present a novel construction of the discriminator in the Reproducing Kernel Hilbert Space (RKHS), 2) theoretically relate the error probability bound of the KL estimates to the complexity of the discriminator in the RKHS space, 3) present a scalable way to control the complexity (RKHS norm) of the discriminator for a reliable estimation of KL divergence, and 4) prove the consistency of the proposed estimator. In three different applications of KL divergence -- estimation of KL, estimation of mutual information and Variational Bayes -- we show that by controlling the complexity as developed in the theory, we are able to reduce the variance of KL estimates and stabilize the training.
% We first present a theoretical analysis based on statistical bounds provide a remedy for this problem by . This enables us to leverage sample complexity and . We then present a scalable way to control the complexity of the discriminator for a reliable estimation of KL divergence. We support our theory and method through controlled experiments in three different scenarios: estimation of KL, estimation of mutual information and Variational Bayes.
\end{abstract}

\section{Introduction}
\label{introduction}
% \lw{[Is the citation in parenthesis a style of NeurIPS? If not you should remove them]}
Estimating {Kullback–Leibler} (KL) divergence from data samples is 
{an essential component in many} machine learning problems including Bayesian inference, calculation of mutual information or methods using information theoretic objectives. Variational formulation of Bayesian Inference requires KL divergence computation, which could be challenging when we only have finite samples from two distributions. Similarly, computation of information theoretic objectives like mutual information requires computation of KL divergence between the joint and the product of marginals. 
% Moreover, since these quantities are often used with neural networks and utilize large data, we need accurate and reliable way to estimate KL divergence from samples and in a scalable way.

KL divergence estimation from samples was studied thoroughly by Nguyen et al. \cite{nguyen2010estimating} using a variational technique, convex optimization and RKHS norm regularization, while also providing theoretical guarantees and insights. However, their technique requires handling the whole dataset at once and is not scalable.
Many modern models need to use KL divergence with large scale data, and often with neural networks,
for example total correlation variational autoencoder (TC-VAE) \cite{chen2018isolating}, adversarial variational Bayes (AVB) \cite{mescheder2018gan},
information maximizing GAN (InfoGAN) \cite{chen2016infogan}, 
and amortized MAP \cite{sonderby2016amortised} all need to compute KL divergence in a deep learning setup. These large scale models 
have imposed new requirements on KL divergence estimation
like \textit{scalability} (able to handle large amount of data samples) and \textit{minibatch compatibility} (compatible with minibatch-based optimization).
    %{that underlies} many modern applications {such as } AVB {and} TC-VAE  
    % and allow optimization of the rest of the network.

% The seminal work of Nguyen et al.\cite{nguyen2010estimating} falls short in the wake of these modern needs.
%to deal with the huge amount of data. 
%{This need} 
Methods like Nguyen et al. \cite{nguyen2010estimating} are not suitable in the large scale setup. These modern needs were later met by modern 
neural network based methods %using neural network as a discriminator to estimate KL divergence 
such as variational divergence minimization (VDM) \cite{nowozin2016f}, mutual information neural estimation (MINE) \cite{belghazi2018mutual}, and discriminator based KL estimation with GAN-type objective \cite{Mescheder2017ICML, sonderby2016amortised}.
A key attribute of these methods is that they are based on updating a neural-net based discriminator
to estimate KL divergence 
from a subset of samples making them scalable and minibatch compatible. 
%The value of KL divergence estimate, therefore, reflects the effect of all the samples since it depends on the optimum discriminator function. \lw{[the point of this last sentence is not too clear]} 
%In our experiments, however, 
We, however, noticed that even in {simple} examples, these methods exhibited pathologies like unreliability (high fluctuation of estimates) or instability during training (KL estimates blowing up). % in different trials, 
% as in GAN type approach by \cite{Mescheder2017ICML}, MINE and VDM).
%MINE and VDM were found to be unstable while the GAN based approach (\cite{Mescheder2017ICML}), although stable, was found to fluctuate heavily {among} different trails. 
% This behavior exacerbated when increasing the size of 
%the hidden layers of 
% the discriminator. 
Similar {observations of} instability of VDM and MINE  have also been reported in the literature \cite{Mescheder2017ICML,song2019understanding}.
%\lw{[The above observations -- have they been mentioned in literature, or only observed by you?]}

Why are these techniques unreliable? In this paper, we 
{attempt to understand} the core problem in the KL estimation using
{discriminator network.} We look at it
from the perspective of statistical learning theory and discriminator function space complexity and draw insights. Based on these insights, we propose that these fluctuations are a consequence of not controlling the smoothness and the complexity of the discriminator function space. Measuring and controlling the complexity of function space itself becomes a difficult problem when the discriminator is a deep neural network. 
% This direction has not been explored in existing works, and it faces the {open question} of how to properly 
% measure the complexity of the large function space represented by neural networks. 
Note that naive approaches to bound complexity by the number of parameters would neither be guaranteed to yield meaningful bound \cite{zhang2016understanding}, nor be easy to implement.
% because it requires dynamically changing the size of the network during optimization. 
%-- {the} large fluctuation in the estimates -- from the perspective of sample complexity. Whenever we are dealing with the estimation of some quantity from finite samples, we need to consider the relation between the complexity of the function and the number of samples. Our proposition {in this paper} is that these fluctuations are a consequence of not controlling the complexity of the discriminator function.  Unfortunately, none of the modern approaches seem to explore this direction. Even if we try to compute and control the complexity of the discriminator function, we face problems because the function space represented by neural network is too large and it is not clear how to estimate good complexity measure of a neural network. 
%Naive approach to bound complexity by the number of parameters would not be very helpful in controlling the complexity because i) in our experiments, we obtain large fluctuation even for the small sized neural networks, ii) we would have to dynamically change the size of the network to control its complexity during optimization. 
% \lw{[I delete the first reason because it is an observation, not a reason.]}

%To provide a remedy for this challenge,
% Therefore, we introduce the following contributions 
% to resolve this challenge. 
% %we make several contributions towards teh resolution of these issues. 
% First, 
% to be able to compute 
%facilitate computation of 
% the complexity of the discriminator function space,
Therefore, we present the following contributions to resolve these challenges. First, we propose 
{a novel construction of}
 the discriminator function using deep network such that it lies in a smooth function space, the Reproducing Kernel Hilbert Space(RKHS). 
By utilizing the learning theory and the complexity analysis of the RKHS space, we 
bound the probability of the error of KL-divergence estimates 
in terms of the radius of RKHS ball and kernel complexity. %Then we 
%Based on this, we then 
% This enables us to theoretically substantiate our main proposition 
% that not controlling the complexity of the discriminator 
% may lead to high fluctuation in estimation. 
%We use ideas from sample complexity analysis and mean embedding of RKHS to develop this theoretical understanding. 
Using this bound, we propose a scalable way to control the complexity by penalizing the RKHS norm. This additional regularization of the complexity is still linear, (${O(m)}$) in time complexity with the number of data samples. Then, we prove consistency of the proposed KL estimator using ideas from empirical process theory. Experimentally, we 
demonstrate that  
the proposed way of controlling complexity significantly improves KL divergence estimation and significantly reduce the variance. In mutual information estimation, our method is competitive with the state-of-the-art method and in Variational Bayesian application, our method stabilizes training of MNIST dataset leading to sharp reconstruction.

\section{Related Work}
\vspace{-.2cm}
%\lw{copied from intro for now..}
Nguyen et al. \cite{nguyen2010estimating} used variational method to estimate KL divergence from samples of two distribution using convex risk minimization (CRM). They used the RKHS norm as a way 
to both measure and penalize
%measure of complexity and penalized it to control 
the complexity of the variational function. However, their work required handling all data at once and solving a convex optimization problem which has time complexity in the order of $O(m^3)$ and space complexity in the order of $O(m^2)$ . Ahuja \cite{ahuja2019estimating} used similar convex formulation in RKHS space and found it difficult to scale.
%The key principle in f-GAN (\cite{nowozin2016f}) is variational divergence minimization (VDM) which is very close to the approach of (\cite{nguyen2010estimating}). However, VDM used neural network and adversarial optimization to make the estimation scalable. Unlike CRM, however, VDM did not control the complexity and resulted into unstable estimation.
%^Variational divergence minimization (
VDM reformulated the f-Divergence objective using Fenchel duality and used a neural network to represent the variational function \cite{nowozin2016f}.
Although close in concept to %the approach of 
\cite{nguyen2010estimating}, it is scalable since it uses a separate discriminator network and adversarial optimization. 
%Unlike CRM, %however, 
%VDM 
It, however, did not control the complexity of the neural-net function, and faced issues with stability.

%MINE gives a tighter bound and uses a discriminator function over which the lower bound in optimized. 
One area of modern application of KL-divergence estimation is in computing mutual information, which is useful in applications such as stabilizing GANs \cite{belghazi2018mutual}.  
MINE \cite{belghazi2018mutual}  also optimized a lower bound to KL divergence (Donsker-Varadhan representation). Similar to VDM, MINE used a neural network as the dual variational function: it is thus scalable, but without complexity control and is unstable.
Another use of KL divergence is scalable variational inference (VI) as shown in AVB \cite{Mescheder2017ICML}.
VI requires KL divergence estimation between the posterior and the prior, which becomes nontrivial when a sample based scalable estimation is required.
%While in some cases, %like vanilla VAE \cite{kingma2013auto}, 
%the posterior distribution could be assumed to have simplified Gaussian distribution to make this computation easier \cite{kingma2013auto}, 
% an expressive posterior distribution 
% is used and requires . 
AVB solved it using GAN-type adversarial formulation and a neural network discriminator. Similarly, \cite{sonderby2016amortised} used GAN-type adversarial formulation to obtain KL divergence in amortized inference.
%GAN based approach uses an objective that is identical to that of GAN (\cite{goodfellow2014generative}) and uses the descriminator to estimate the KL divergence.

Chen et al. \cite{chen2018isolating} proposed TC-VAE to improve disentanglement by penalizing the KL divergence between the marginal latent distribution and the product of marginals in each dimension. The KL divergence was computed by a minibatch-based sampling strategy that gives a biased estimate. Our work is close to Song et al. \cite{song2019understanding} who investigated the high variance in existing mutual information estimators and found that clipping the discriminator output is helpful in reducing variance. In our work, we take a principled way to connect variance to the complexity of discriminator function space and constrain it by penalizing its RKHS norm instead.
None of the existing works considered looking at the discriminator function space, connecting its complexity to the
unreliable KL-divergence estimation, 
%the problem of unreliable KL-divergence estimates, its  
or mitigating the problem by controlling the complexity.
%While unstable or inconsistent estimates have been observed for most existing approaches to KL-divergence estimation, the theoretical underpinning has not been understood and none of the existing works have considered controlling the complexity of the discriminator function for the purpose of a more consistent KL estimation.} 

\section{Reproducing Kernel Hilbert Space}
% \textbf{Reproducing Kernel Hilbert Space}:
Let $\mathcal{H}$ be a Hilbert space of functions $f:\mathcal{X}\to {\rm I\!R}$ defined on non-empty space $\mathcal{X}$. It is a Reproducing Kernel Hilbert Space (RKHS) if the evaluation functional, $\delta_x :\mathcal{H} \to {\rm I\!R}$, $\delta_x :f \mapsto f(x)$, is linear continuous $\forall x \in \mathcal{X}$. Every RKHS, $\mathcal{H}_K$, is associated with a unique positive definite kernel, $K: \mathcal{X}\times \mathcal{X}\to {\rm I\!R}$, called the reproducing kernel \cite{berlinet2011reproducing}, such that it satisfies:\\
% \begin{enumerate}
    \hspace{0.5cm}
    1. $\forall x \in \mathcal{X}, K(.,x) \in \mathcal{H}_K$  \hspace{0.5cm}
    2. $\forall x \in \mathcal{X}, \forall f \in \mathcal{H}_K,\hspace{.1cm} \langle f,K(.,x) \rangle_{\mathcal{H}_K}=f(x)$ \hspace{.2cm} 
    
% \end{enumerate}
%Reproducing Kernel Hilbert Spaces 
RKHS is often studied using a specific integral operator. Let $\mathcal{L}_2(d\rho)$ be a space of functions $f: \mathcal{X} \to {\rm I\!R}$ that are square integrable with respect to a Borel probability measure $d\rho$ on $\mathcal{X}$, we define an integral operator  $\mathscr{L}_K:\mathcal{L}_2(d\rho) \to \mathcal{L}_2(d\rho)$ \cite{bach2017equivalence, cucker2002mathematical}:
% \begin{align}
$   (\mathscr{L}_K f)(x)=\int_{\mathcal{X}}f(y)K(x,y)d\rho(y) $
% \end{align}
This operator will be important in constructing a function in RKHS and in computing sample complexity.

% \textbf{Mean Embedding in RKHS}:
% Let $f:\mathcal{X}\to {\rm I\!R}$ be a function in {RKHS}, $\mathcal{H}_K$, and $p$ be a Borel probability measures on $\mathcal{X}$. If $E_{x\sim p}\sqrt{K(x,x)}<\infty$, then we have $\mu_p \in \mathcal{H}_K$ called the mean embedding of the distribution $p$ and defined as \cite{sriperumbudur2010hilbert,berlinet2011reproducing,gretton2012kernel}: 
% \begin{align}
% $    E_{x\sim p}f=\langle f,\mu_p \rangle_{\mathcal{H}_K}$
% \end{align}
% The condition for the existence of mean embedding is readily satisfied since we assume %stronger condition:
% $\underset{x,t}{sup}\hspace{0.1cm}{K(x,t) < \infty}$. 

\section{Problem Formulation and Contribution}
%\subsection{GAN-type Optimization to Estimate KL Divergence}

\textbf{GAN-type Objective for KL Estimation:}
Let $p(x)$ and $q(x)$ be two probability density functions %\footnote{We do not need that two distributions admit density functions, all we need is samples from two distributions. Here, we assume density for the sake of simplicity and consistency with the literature.} 
in space $\mathcal{X}$ and we want to estimate their KL divergence using finite samples from each distribution  
%To estimate the KL divergence 
in a scalable and minibatch compatible manner. 
As shown in \cite{Mescheder2017ICML, sonderby2016amortised}, this can be achieved by using a discriminator function. First, a discriminator $f:\mathcal{X}\to {\rm I\!R}$ is trained with the objective:
\begin{align}
\label{gan_kl}
    % E_{p(x)}\log \sigma(f(x))+E_{q(x)}\log (1-\sigma(f(x)))\\
    f^*=\underset{f}{\operatorname{argmax}}[{E_{p(x)}\log \sigma(f(x))+E_{q(x)}\log (1-\sigma(f(x)))}] %\vspace{-.2cm}
\end{align}
where $\sigma$ is the Sigmoid function given by $\sigma(x)=\frac{e^x}{1+e^x}$.
Then it can be shown \cite{Mescheder2017ICML, sonderby2016amortised} that the KL divergence $KL(p(x)||q(x))$ is given by:
% \begin{align}
$
\label{kl_expectation}
 KL(p(x)||q(x))=E_{p(x)}[f^*(x)]  
$
% \end{align}
%\subsection{Sources of Error}

\textbf{Sources of Error:}
Eq. (\ref{gan_kl}) is ambiguous in the sense that it is silent about the discriminator function space over which the optimization is carried out.
Typically, a neural network is used as the discriminator. This implies that we are considering the space of functions represented by the neural network of given architecture as the hypothesis space, over which the maximization occurs in eq. (\ref{gan_kl}). Hence, we must rewrite eq. (\ref{gan_kl}) as 
\begin{align}
\label{finite_opt}
    f^*_h=\underset{f \in h}{\operatorname{argmax}}[{E_{p(x)}\log \sigma(f(x))+E_{q(x)}\log (1-\sigma(f(x)))}] %\vspace{-.2cm}
\end{align}
where $h$ is the discriminator function space. Furthermore, we also approximate integrals in eq. (\ref{finite_opt}) with the Monte Carlo estimate using finite number of samples, say $m$, from the distribution $p$ and $q$. 
\begin{align}
\label{kl_h}
    f^m_h=\underset{f \in h}{\operatorname{argmax}}\Big[{\frac{1}{m}\sum_{x_i \sim p(x_i)}\log \sigma(f(x_i))+\frac{1}{m}\sum_{x_j \sim q(x_j)}\log (1-\sigma(f(x_j)))}\Big] 
\end{align}
Similarly, we write KL estimate obtained from, respectively, infinite and finite samples as:
\begin{align}
\label{kl_finite}
 KL(f)=E_{p(x)}[f(x)],   \hspace{0.4cm}  KL_m(f)=\frac{1}{m}\sum_{x_i \sim p(x_i)}[f(x)]
\end{align}
Each of these steps introduce some error in our estimate. We can now start our analysis by first decomposing the total estimation error as:
\begin{align}
\label{total_error}
KL_m(f^m_h)-KL(f^*)=\underbrace{KL_m(f^m_h)-KL(f^m_h)}_{\text{Deviation-from-mean error}}+\underbrace{KL(f^m_h)-KL(f^*_h)}_{\text{Discriminator induced error}}+\underbrace{KL(f^*_h)-KL(f^*)}_{Bias}
\end{align}
%In the experiment section, we show that if this hypothesis space is too large, we will obtain highly fluctuating results.  
%To facilitate our analysis of the variance of KL estimate in terms of the complexity of hypothesis space,
This equation decomposes total estimation error into three terms: 1) deviation from the mean error, 2) 
error in KL estimate 
by %because of erroneous estimate of 
the discriminator %caused by 
due to using finite samples in optimization eq. (\ref{kl_h}), and 3) bias when the considered function space %is limited and 
does not contain the optimal function.
Here, we concentrate on quantifying the probability of deviation-from-mean error which is directly related to observed variance of the KL estimate. 
%\lw{[ put a short name to each term (similar to how TC-VAE did their equation breakdown)]}
%Therefore, towards the direction of controlling the complexity of the (hypothesis) function space, 
%Below, we propose a way to achieve discriminator function lying in RKHS by slightly modifying the neural network.

%\subsection{Overview of Contributions}
\textbf{Summary of Technical Contributions:}
Since the deviation is the difference between a random variable and its mean, we can bound the probability of this error using concentration inequality and the complexity of the function space of $f^m_h$. To use smooth function space, we propose to construct a function out of neural networks such that it lies on RKHS (Section \ref{construct_rkhs}). Then, we bound the probability of deviation-from-mean error through the covering number of the RKHS space (Section \ref{bounding_error}), then control complexity (Section \ref{fitting_pieces}) and prove consistency of the proposed estimator (Section \ref{section_consistency}).
% \begin{align}
%     f^m_h=\underset{f \in h}{\operatorname{argmax}}\Big[{\frac{1}{m}\sum_{x_i \sim p(x_i)}\log \sigma(f(x_i))+\frac{1}{m}\sum_{x_j \sim q(x_j)}\log (1-\sigma(f(x_j)))} + \lambda ||g||^2\Big] 
% \end{align}

\section{Constructing $f$ in RKHS}
\label{construct_rkhs}
% Since KL is obtained by expectation operation, we can bound the probability of variance by using idea of uniform bound  from statistical learning theory \cite{vapnik} and complexity analysis\cite{cucker2002mathematical}. This technique, however, depend on a measure of complexity of the function space. 
% \sout{However, measuring the complexity of neural network is not clear.} To overcome this issue, we alternatively construct an RKHS space from the neural network and quantify its complexity.
% To construct a function in RKHS, we use an operator $T$ related to integral operator $\mathscr{L}_K$ by $\mathscr{L}_K=TT^*$ \cite{bach2017equivalence}. 
The following theorem due to \cite{bach2017breaking} paves a way for us to construct a neural function in RKHS.
%function in RKHS via a simple modification to the neural network.

\begin{theorem}{[\cite{bach2017breaking} Appendix A]}
\label{rkhs_construct}
A function $f \in \mathcal{L}_2(d\rho)$ is in Reproducing Kernel Hilbert Space, $\mathcal{H}_{K}$, if and only if it can be expressed as
\begin{align}
    \forall x \in \mathcal{X}, f(x)=\int_{\mathcal{W}}g(w)\psi(x,w)d\tau(w),
\end{align}{}
for a certain function $g: \mathcal{W}\to \mathbb{R}$ such that $||g||^2_{\mathcal{L}_2(d\tau)} < \infty$. The RKHS norm of $f$ satisfies 
$
 ||f||^2_{\mathcal{H}_{K}} \leq  ||g||^2_{\mathcal{L}_2(d\tau)}
$
and the kernel $K$ is given by
\begin{align}
    \label{kernel}
    K(x,t)=\int_{\mathcal{W}}\psi(x,w)\psi(t,w)d\tau(w)
\end{align}

\end{theorem}
Theorem \ref{rkhs_construct} not only gives us a condition when a square integrable function is guaranteed to lie in RKHS, it also provides us with a recipe to construct a function in RKHS. We use this theorem with the neural networks as $\psi$ and $g$.
We sample $w  \sim \mathcal{N}(0, \gamma\text{I})$ and pass it through two neural networks, $\psi$ and $g$, where $\psi$ takes $x$ and $w$ as two arguments and $g$ takes only $w$ as an argument.
More precisely, we consider $\psi(x,w)=\phi_{\theta}(x)^Tw$. 
% where $\phi_{\theta}(x)$ denotes neural network transformation until the last layer, and $w$ is the last linear layer sampled from Gaussian distribution. While in principle any layer could be made stochastic, we chose this architecture to reduce the computational cost of sampling. 
The kernel $K$, as defined in eq. (\ref{kernel}), can be obtained as:
\begin{align}
\label{define_K}
     K_{\theta}(x^*,t^*)&=\int_{\mathcal{W}}\phi_{\theta}(x^*)^Tww^T\phi_{\theta}(t^*) d\tau(w)
     =\gamma \phi_{\theta}(x^*)^T\phi_{\theta}(t^*) 
\end{align}
where $E_{w \sim \mathcal{N}(0,\gamma \text{I})}[ww^T] = \gamma \text{I}$. %Also, 
We sometimes denote the kernel $K$ by $K_{\theta}$ to emphasize that it is a function of neural network parameters, $\theta$.

Traditionally, kernel $K$ remains fixed and the norm of the function $f$ determines the complexity of the function space. 
% For example, \cite{nguyen2010estimating} penalized the $||f||_{\mathcal{H}_K}$ as a way to control the function space while estimating KL divergence. 
% {In our RKHS formulation of neural networks,} {the nature of the problem has changed:} $||f||_{\mathcal{H}_K}$ cannot increase beyond 1, but
%Rather, 
In our formulation, both the RKHS kernel and its norm with respect to the kernel change during training since the kernel depends on neural network parameters, $\theta$. Therefore, 
{the challenge is}
to tease out how neural {parameters, $\theta$,} 
affect the deviation-from-mean error in  eq. (\ref{total_error}).
% how we can control that complexity during training of the neural network.

\section{Error Analysis and Control}
\textbf{Assumptions:} Before starting our analysis, we list assumptions upon which our theory is based. \\
\hspace{0.2cm} A1. The input domains $\mathcal{X}$ and $\mathcal{W}$ are compact. \\
\hspace{0.2cm} A2. The functions $\phi_{\theta}$ and $g$ are Lipschitz continuous with Lipschitz constants $L_{\phi}$ and $L_g$ respectively.\\
\hspace{0.2cm} A3. Higher order derivatives $D_x^{\alpha}K(x,t)$ up to some high order $\tau =h/2$ of kernel $K$ exist.

Assumptions A1 is satisfied in our experiments since we consider a bounded set in $\mathbb{R}^n$ and $\mathbb{R}^D$ as our domains. Similarly, A2 is satisfied since we enforce Lipschitz continuity of $\phi$ and $g$ by using spectral normalization \cite{miyato2018spectral}. Assumption A3 is a bit subtle. By the definition of $K$ in eq.(\ref{define_K}), higher order derivative of $K$ exists iff higher order derivative of $\phi_{\theta}$ exists. This is readily satisfied by deep networks with smooth activation functions, and is true everywhere except at origin for ReLU activation. Using the boundedness of the input domain and Lipschitz continuity, we show the following:

\begin{proposition}
Under the assumptions A1, A2, we have $\underset{K_{\theta}}{sup}$ \hspace{0.1cm} $K_{\theta}(x,t) < \infty$ and 
 $||g||^2_{\mathcal{L}_2(d\tau)} < \infty$.
\end{proposition}

\subsection{Bounding the Error Probability of KL Estimates}
\label{bounding_error}
Bounding the probability of deviation-from-mean error (eq. (\ref{total_error})) is tricky since, in our case, the kernel is not fixed and we are also optimizing over them. We bound it in two steps: 1) we derive a bound for a fixed kernel, 2) we take supremum of this bound over all the kernels parameterized by $\theta$.

For a fixed kernel, we first bound the probability of deviation-from-mean error 
in terms of the covering number 
in Lemma \ref{sample_complexity}. 
%To identify how neural network contribute to this error, 
%To obtain the probability of deviation-from-mean error, in Lemma \ref{sample_complexity}, we first bound this error probability in terms of the covering number. 
%Then, 
We then use an estimate of the covering number of RKHS due to \cite{cucker2002mathematical} to 
relate the bound %obtain the bound 
%of error probability 
%in terms of 
to kernel $K_{\theta}$ in Theorem \ref{complexity}, 
identifying the role of neural networks in this error bound. 
%Our main proposition is that, if we do not control the complexity of the discriminator hypothesis space, it will lead to high fluctuation in the estimates of the KL divergence. Towards justifying this objective, 

% In this section, 
% we investigate how the neural network enters the equation of the error bound of the estimates of KL divergence, 
% %Once we have a clear picture of the factors affecting the error bound, we will discuss how to control the complexity of the discriminator in a scalable way. 
% focusing on the first term of eq.(\ref{total_error}). 
%  By following \cite{cucker2002mathematical}, we can obtain a bound on the probability of this error \lw{[Does this sentence mean that what you do below is following [10]?] --that's the impression this sentence leaves, even though I think you simply means that you're basing off that theory. YOu need to make it more clear what are your basis and what are your ideas.]}. 

\begin{lemma}
\label{sample_complexity}
Let $f^m_{\mathcal{H}_K}$ be the optimal discriminator function in an RKHS $\mathcal{H}_{K}$ which is M-bounded. Let ${KL}_m(f^m_{\mathcal{H}_K})=\frac{1}{m}\sum_i f^m_{\mathcal{H}_K}(x_i)$ and $KL(f^m_{\mathcal{H}_K}) = E_{ p(x)}[f^m_{\mathcal{H}_K}(x)]$ be the estimate of KL divergence from m samples and that by using the true distribution $p(x)$ respectively.
% \begin{align}
%     \hat{KL}_m=\frac{1}{m}\sum_i f^*(x_i), and \hspace{0.2cm}
%     KL = E_{ p(x)}[f^*(x)]
% \end{align}
Then the probability of error at some accuracy level, $\epsilon$, is lower-bounded as:
\begin{align}
\nonumber
    \text{Prob.}(&|{KL}_m(f^m_{\mathcal{H}_K})-{KL}(f^m_{\mathcal{H}_K})|\leq \epsilon) 
    \geq 1-2\mathcal{N}(\mathcal{H}_K, \frac{\epsilon}{4\sqrt{S_K}})\exp(-\frac{m\epsilon^2}{4M^2})
\end{align}
where $\mathcal{N}(\mathcal{H}_K,\eta)$ denotes the covering number of an RKHS space $\mathcal{H}_K$ with disks of radius $\eta$, and $S_K=\underset{x,t}{sup} $\hspace{0.1cm} ${K(x,t)}$ which we refer to as kernel complexity.
\vspace{-0.1cm}
\end{lemma}
\begin{proof}[Proof Sketch]
We cover RKHS with discs of radius $\eta=\frac{\epsilon}{4\sqrt{S_K}}$. Within this radius, the deviation does not change too much. So, we can bound deviation probability at the center of disc and apply union bound over all the discs. To bound deviation probability at the center, we apply Hoeffding's inequality and applying union bound simply leads to counting number of discs which is exactly the covering number. See supplementary materials for the full proof.
\end{proof}
\vspace{-0.2cm}
Lemma \ref{sample_complexity} bounds the probability of error in terms of the covering number of the RKHS space. Note that the radius of the disc is inversely related to $S_K$ which % --referred as the kernel complexity-- intuitively 
indicates how complex the RKHS space defined by the kernel $K_{\theta}$ is. %In our case, 
Here $K_{\theta}$ depends on the neural network parameters $\theta$. Therefore, we 
denote $S_K$ as a function of $\theta$ as $S_K(\theta)$ and 
term it kernel complexity. Next, we use Lemma 2 due to \cite{cucker2002mathematical} %gives the covering number of the RKHS space, which we use 
to obtain an error bound in estimating KL divergence with finite samples in Theorem \ref{complexity}.

\begin{lemma}[\cite{cucker2002mathematical}]
\label{covering number}
Let $K: \mathcal{X}\times \mathcal{X}\to {\mathbb{R}} $ be a $\mathcal{C}^\infty$ Mercer kernel and the inclusion $I_K:\mathcal{H}_K\xhookrightarrow{}\mathcal{C}(\mathcal{X})$ be the compact embedding defined by $K$ to the Banach space $\mathcal{C}(\mathcal{X})$. Let $B_R$ be the ball of radius $R$ in RKHS $\mathcal{H}_{K}$. Then $\forall \eta>0, R >0, h>n $, we have
\begin{align}
    \ln \mathcal{N}(I_K(B_R), \eta) \leq \left( \frac{RC_h}{\eta} \right)^{\frac{2n}{h}}
\end{align}
where $\mathcal{N}$ gives the covering number of the space $I_K(B_R)$ with discs of radius $\eta$, and $n$ represents the dimension of the input space $\mathcal{X}$. $C_h$ is given by 
$
C_h=C_s\sqrt{||\mathscr{L}_s||}
$
where $\mathscr{L}_s$ is a linear embedding from square integrable space $\mathcal{L}_2(d\rho)$ to the Sobolev space $H^{h/2}$ and $C_s$ is a constant.
\end{lemma}
To prove Lemma \ref{covering number} \cite{cucker2002mathematical}, the RKHS space is embedded in the Sobolev Space $H^{h/2}$ using $\mathscr{L}_s$ and then the covering number of the Sobolev space is used. Thus the norm of $\mathscr{L}_s$ and the degree of Sobolev space, $h/2$, appears in the covering number of a ball in $\mathcal{H}_K$. 
In Theorem \ref{complexity}, we use Lemma \ref{sample_complexity} and \ref{covering number} to bound the {estimation error of KL divergence}.

\begin{theorem}
\label{complexity}
Let ${KL}(f^m_{\mathcal{H}})$ and ${KL}_m(f^m_{\mathcal{H}})$ be the estimates of KL divergence obtained by using true distribution $p(x)$ and $m$ samples respectively as described in Lemma \ref{sample_complexity}, then the probability of error in the estimation at the error level $\epsilon$ is given by:
\begin{align*}
   \text{Prob.}(&|{KL}_m(f^m_{\mathcal{H}})-{KL}(f^m_{\mathcal{H}})|\leq \epsilon) \geq 1-2\exp\Bigg[\left( \frac{4RC_p\sqrt{S_p||\mathscr{L}_p||}}{\epsilon} \right)^{\frac{2n}{h}}-\frac{m\epsilon^2}{4M^2}\Bigg]
\end{align*}
where $C_p\sqrt{S_p||\mathscr{L}_p||} = \underset{K_{\theta}}{sup}  C_s\sqrt{S_K(\theta)||\mathscr{L}_s||}$, i.e. $C_p, S_p, \mathscr{L}_p$ correspond to a kernel for which the bound is maximum.
\vspace{-0.2cm}
\end{theorem}
\begin{proof}
We prove this in two steps: First we obtain an error bound for a fixed kernel space and apply supremum over all $\theta$. For any RKHS $\mathcal{H}_{K_{\theta}}$, with fixed kernel $K_{\theta}$, we have
\begin{align}
\label{fixed_rkhs}
   \text{Prob.}(&|{KL}_m(f^m_{\mathcal{H}_{K_{\theta}}})-{KL}(f^m_{\mathcal{H}_{K_{\theta}}})|\geq \epsilon) \leq 2\exp\Bigg[\left( \frac{4RC_s\sqrt{S_K(\theta)||\mathscr{L}_s||}}{\epsilon} \right)^{\frac{2n}{h}}-\frac{m\epsilon^2}{4M^2}\Bigg]
\end{align}
We prove this error bound as follows. Lemma \ref{covering number} gives the covering number of an RKHS ball of radius $R$, which we apply to Lemma \ref{sample_complexity}. We fix the radius of discs to $\eta=\frac{\epsilon}{4\sqrt{S_K}}$ in Lemma \ref{sample_complexity} and substitute $C_h=C_s\sqrt{||\mathscr{L}_s(\theta)||}$ to obtain eq.(\ref{fixed_rkhs}).

Since we are continuously changing $\theta$ during training, the kernel also changes. Hence, to find the upper bound over all possible kernels, we take the supremum over all kernels.
\begin{align}
\label{sup_K}
   \text{Prob.}(|{KL}_m(f^m_{\mathcal{H}})-{KL}(f^m_{\mathcal{H}})|\geq \epsilon) &\leq \underset{K_{\theta}}{sup} \hspace{0.2cm} \text{Prob.}(|{KL}_m(f^m_{\mathcal{H}_{K_{\theta}}})-{KL}(f^m_{\mathcal{H}_{K_{\theta}}})|\geq \epsilon)\\
   \label{theta_p}
   & \leq 2\exp\Bigg[\left( \frac{4RC_p\sqrt{S_p||\mathscr{L}_p||}}{\epsilon} \right)^{\frac{2n}{h}}-\frac{m\epsilon^2}{4M^2}\Bigg]
\end{align}
where $S_p = S_K(\theta_p)$ and $\mathscr{L}_p = \mathscr{L}_K(\theta_p)$, \textit{i.e.,} $S_p$ and $\mathscr{L}_p$ correspond to kernel complexity and Sobolev operator norm corresponding to optimal kernel $K_{\theta_p}$ that extremizes eq. (\ref{sup_K}). Theorem statement readily follows from eq. (\ref{theta_p})
\vspace{-0.3cm}
\end{proof}
Theorem \ref{complexity} shows that the error increases exponentially with the radius of the RKHS space, $R$, complexity of the kernel $S_K(\theta_p)$, and the norm of the Sobolev space embedding operator $||\mathscr{L}_p||$. The Sobolev embedding operator,
% is closely related to the operator $\mathscr{L}_K$. The difference is that
$\mathscr{L}_p$, is a mapping from $\mathcal{L}_2(d\rho)$ to the Sobolev space $H^{h/2}$. It can be shown \cite{cucker2002mathematical} that the operator norm can be bounded as
% \begin{align}
% \label{sobolev_operator}
$
    ||\mathscr{L}_p|| \leq \rho(\mathcal{X})\sum_{|\alpha|\leq h/2}\underset{x,t \in \mathcal{X}}{sup}(D^{\alpha}_x K_{\theta_p}(x,t))^2, 
$
% \end{align}
where $\rho$ is the measure of the input space $\mathcal{X}$.
Therefore, the norm $||\mathscr{L}_p||$ directly measures smoothness of $K_{\theta_p}$ in terms of norm of its derivative in addition to the supremum value of $K$, while $S_K(\theta_p)$ only depends on the supremum value of $K_{\theta_p}$. 
% Also, note that for $\alpha =0$, $||\mathscr{L}_p||$ is exactly equal to $||S_K||$.
\vspace{-0.1cm}
\subsection{Complexity Control}
\label{fitting_pieces}
From Theorem 2, we see that the error probability could be decreased by decreasing $R, ||\mathscr{L}_p||$ and $S_K(\theta_p)$. 
% $||g||_{\mathcal{L}_2(d\tau)}^2$, 
Using argument similar to the proof of Proposition 1, we can show that the Lipschitz constraint on $\phi_{\theta}$ also affects $S_K$ and may affect $||\mathscr{L}_p||$. In our experiments, however, we fix the Lipschitz constraints during optimization and do not change $S_K$ and $||\mathscr{L}_p||$ dynamically. Here, we focus on the norm, $R$ from Theorem \ref{complexity}. To obtain the optimal discriminator $f^m_h$, we optimize the following objective with an extra penalization of the upper bound, i.e. $||g||$ on the RKHS norm of $f$:
\begin{align}
\label{aug_obj}
  f^m_h=\underset{f \in h}{\operatorname{argmax}} {\frac{1}{m}\sum_{x_i \sim p(x_i)}\log \sigma(f(x_i))+\frac{1}{m}\sum_{x_j \sim q(x_j)}\log (1-\sigma(f(x_j)))} - \frac{\lambda_0}{m}||g||_{\mathcal{L}_2(d\tau)}^2
\end{align}

%Since we do not have a way to estimate $\frac{n}{h}$, we treat $\gamma$ 
 % to be learned. 
The regularization term prevents the radius of RKHS ball from growing, maintaining a low error probability. Optimization of eq. (\ref{aug_obj}) w.r.t. neural network parameters $\theta$ allows dynamic control of the complexity of the discriminator function on the fly in a scalable and efficient way. Note that, computation of $||g||_{\mathcal{L}_2(d\tau)}$ requires randomly sampling $w \sim \mathcal{N}(0, \gamma \textbf{I})$ and passing through neural network $g$ independent of the data $x_i, x_j$. \textit{Therefore, if the computational complexity of optimization is O(m), it will remain the same after incorporating this additional term, i.e. regularization does not increase asymptotic time complexity which is linear with the number of samples, m.}
\begin{figure}[tbp]
\begin{center}
\centerline{\includegraphics[width=0.95\linewidth]{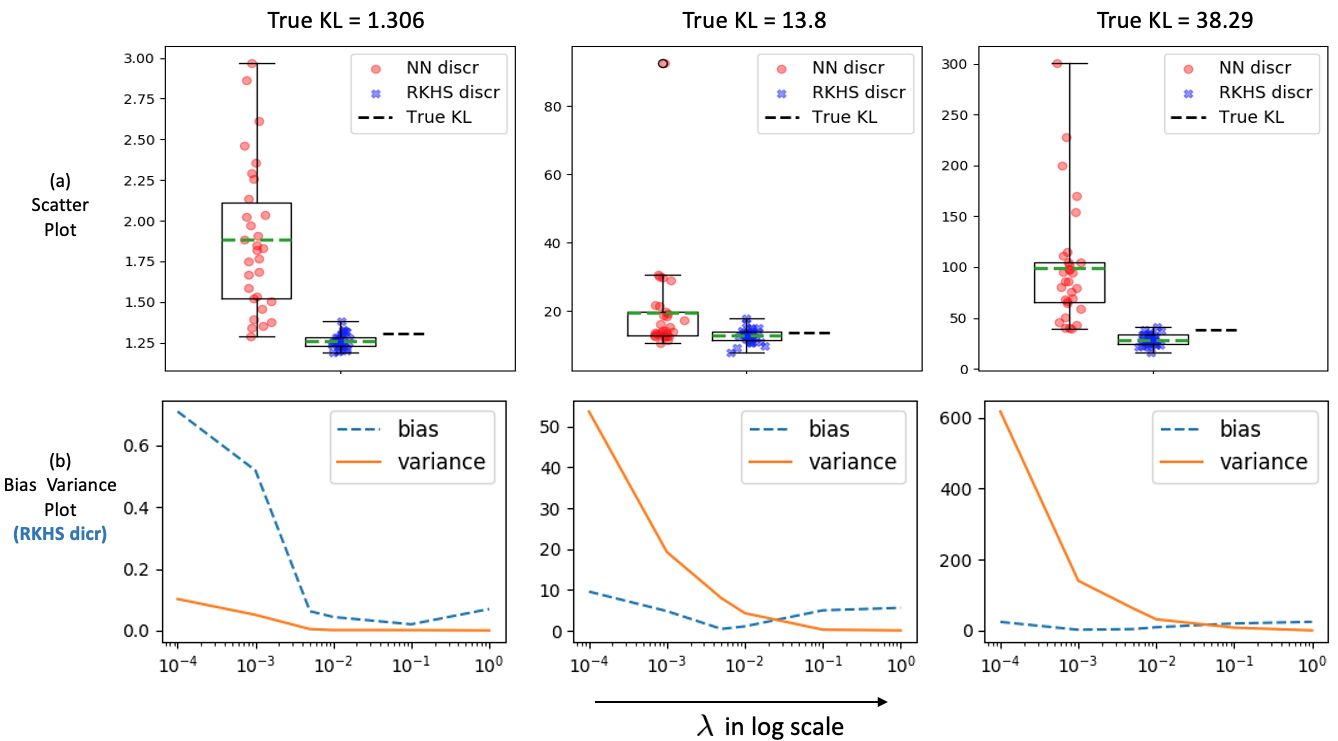}}
\caption{\small{a) Top scatter plot compares KL divergence estimates between a method using Neural network discriminator without complexity control (red) and that using RKHS discriminator with compelxity control (blue); b) In the bottom, we show the effect of varying the regularization parameter $\lambda$ on bias and variance while using the RKHS discriminator with complexity control as in eq.(\ref{aug_obj})}. 
%i) infinite samples, ii) finite samples and a normal neural network discriminator, and iii) finite sample and the presented RKHS discriminator with complexity control.
% \vspace{-.2cm}
}
\label{three_kl}
\end{center}
\vspace{-.7cm}
\end{figure}
\section{Variance and Consistency of the Estimate}
\label{section_consistency}
\subsection{Variance Analysis}
Theorem \ref{complexity} gives an upper bound on the probability of error. Intuitively, the variance and probability of error behave similarly for many distributions, i.e. higher variance might indicate higher probability of error. Below we quantify this intuition for a Gaussian distributed estimate:
\begin{theorem}
\label{variance}
Let $X = KL_m(f^m_{\mathcal{H}})$ be the estimated KL divergence using m samples as described in Theorem \ref{complexity}. Assuming that $X$ follows a Gaussian distribution $X \sim \mathcal{N}(\mu, \sigma)$, we can obtain an upper bound on this variance of the estimate as follows:
\begin{align}
    \sigma &\leq \frac{\epsilon}{\text{erf}^{-1}\Big[-4\exp\Big[\left( \frac{4RC_p\sqrt{S_p||L_p||}}{\epsilon} \right)^{\frac{2n}{h}}-\frac{m\epsilon^2}{4M^2}\Big] + 1 \Big]}
\end{align}
where erf is the Gauss error function and is a monotonic function.
\end{theorem}
Obviously, this relation applies only to Gaussian distributed estimate, a strong assumption. However, Theorem \ref{variance} is presented for illustrative purpose. 
It suggests that by decreasing $R$, the radius of the RKHS ball, the variance of the estimate could be decreased. Experimentally, we observe that the variance decreases as we penalize the RKHS norm more, consistent with the spirit of Theorem \ref{variance}.
\vspace{-0.1cm}
\subsection{Consistency of Estimates}
Here we show that the regularized objective leads to a consistent estimation.
\begin{theorem}
Let $f^*$ and $f^m$ be optimal discriminators as described in eq. (\ref{gan_kl}) and eq. (\ref{aug_obj}) respectively, and the  KL estimate is given by $KL(f)=E_{p(x)}[f(x)],   \hspace{0.4cm}  KL_m(f)=\frac{1}{m}\sum_{x_i \sim p(x_i)}[f(x)]$. Then, in the limiting case as $m \to \infty$, $|KL_m(f^m_h)-KL(f^*)| \to 0$.
\end{theorem}
\begin{proof}[Proof Sketch]
The difference between the true KL divergence and the estimated KL divergence can be divided into three terms as shown in eq. (\ref{total_error}). We assume that our function space is rich enough to contain the true solution, driving bias to zero. From Theorem \ref{complexity}, we see that in the limiting case of $m \to 0$, the deviation-from-mean error goes to $0$.  Therefore, the key step that remains to be shown is that the discriminator induced error (second term in eq.(\ref{total_error})) also goes to 0 as $m \to \infty$.

It can be shown if we can prove that the optimal discriminator in eq. (\ref{aug_obj}) approaches the optimal discriminator in eq. (\ref{finite_opt}). To prove this, we show that the argument being maximized by $f^m_h$ approaches the argument being maximized by $f_h^*$ in the limiting case. To show this, we need to show that the function space, $\log \sigma f$, is Glivenko Cantelli \cite{van1996weak}, which we prove in following steps:\\
1. We show that $f$ is Lipschitz continuous by definition and due to Lipschitz continuity of $\phi_{\theta}$. Then we show that $\log \sigma f$ is Lipschitz continuous if $f$ is Lipschitz continuous.  \\
2. Then we show that for a class of functions with Lipschitz constant $L$, the metric entropy, $\log N$, can be obtained in terms of $L$ and entropy number of the bounded input space, $\mathcal{X}$. \\
3. Since the metric entropy does not grow with the number of samples $m$, we show that $\frac{1}{m} \log N \to 0$ which lets us show that $\log \sigma f$ belongs to Glivenko Cantelli class of functions by using Theorem 2.4.3 from \cite{van1996weak}. See supplementary material for the complete proof.
\end{proof}

\section{Experimental Results}
We present results on three applications of KL divergence estimation: 1. KL estimation between simple Guaussian distributions, 2. Mutual information estimation, 3. Variational Bayes. In our experiments, the RKHS discriminator is constructed with $\psi$ and $g$ networks as described in Section \ref{construct_rkhs}, where the network $\psi$ is very close to a regular neural network. In two experiments, we compare our results with the models using regular neural net discriminator to ensure that the difference in performance between RKHS and regular neural network is not due to architectural difference. 
% In each use case, we investigate the implications of our theoretical analysis in the experiments.

\begin{figure}[tbp]
\begin{center}
\centerline{\includegraphics[width=\linewidth]{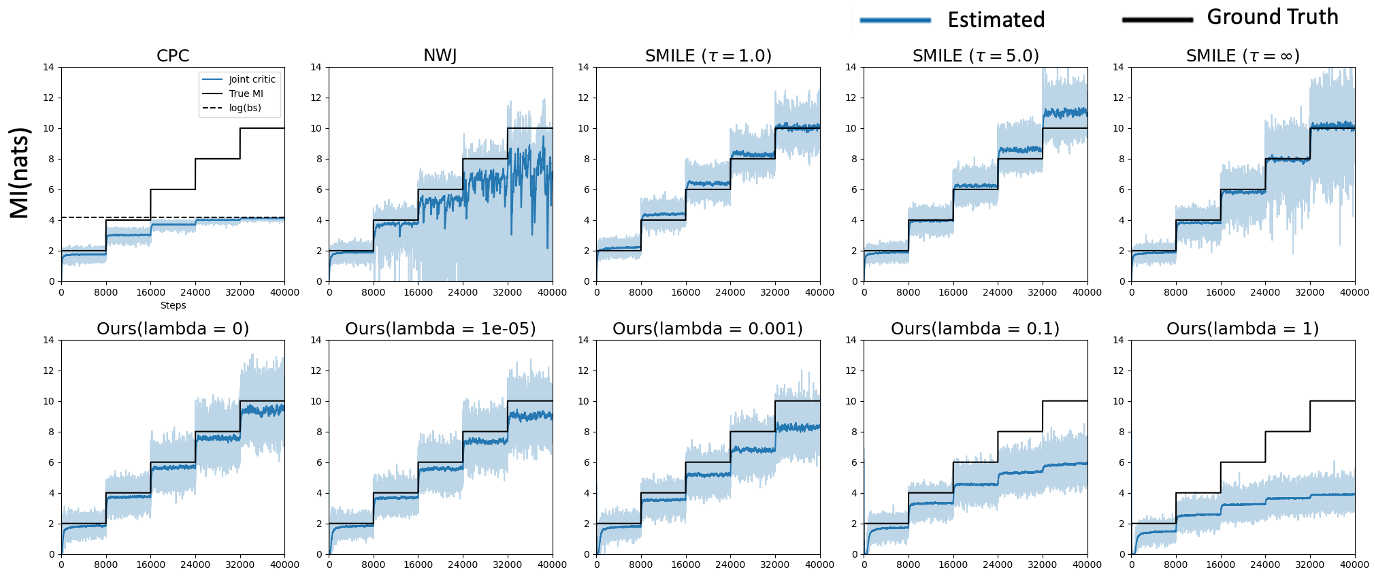}}
\caption{\small{Comparing our method with CPC \cite{oord2018representation}, convex risk minimization(NWJ) \cite{nguyen2010estimating} and SMILE \cite{song2019understanding}} regarding mutual information estimation between two variables. }
%i) infinite samples, ii) finite samples and a normal neural network discriminator, and iii) finite sample and the presented RKHS discriminator with complexity control.
% \vspace{-.2cm}}
\label{mi}
\end{center}
\vspace{-.8cm}
\end{figure}
\paragraph{KL Estimation between Two Gaussians}
% \textbf{Setup:}
We assume that we have finite sets of samples from two distributions. We further assume that we are required to apply minibatch based optimization.  We consider estimating KL divergence between two Gaussian distributions in 2D, where we know the {analytical} KL divergence between the two distributions as the ground truth. We consider three different pairs of distributions corresponding to true KL divergence values of $1.3, 13.8$ and $38.29$, respectively and use $m=5000$ samples from each distribution to estimate KL in the finite case.
% As the discriminator, we use a fully connected neural network with two hidden layers. The number of hidden units are varied to understand the effect of the discriminator complexity {on the fluctuation of the KL estimate}. The dimensions are kept identical between the neural-net discriminator and the RKHS discriminator, the latter being different only in that its last layer is stochastic. 
We repeat the estimation experiments with random initialization 30 times and report the mean, standard deviation, scatter and box plots. 

% \textbf{Finite \textit{v.s.} Infinite Samples:}
% In infinite samples experiment, we assume 
% %access to the model generating data from given distributions, 
% that we can continuously sample 
% %and, 
% %instead of using a finite number of samples, we continuously %sample 
% from the model generating data 
% from the two given distributions. 
% The results of KL estimates using infinite samples is shown in Fig.~\ref{three_kl} and Table.\ref{kl_table} left, % \ref{complexity-control}, 
% in comparison with estimates using finite samples without controlling the complexity of the neural-net discriminator.
% %In the table, we only show result corresponding to one architecture to save space. 
% We observe that when we use infinite samples, we obtain an estimate with low variance and values close to the analytical truth in KL = $1.3$ and KL = $13.8$ and an underestimate when KL = $61.1$. In contrast, when we use finite samples without controlling the complexity of the neural-net discriminator, the estimates fluctuated heavily confirming our hypothesis:  
% %claim that the issue is that of sample complexity. 
% we need to control the complexity of the function when the number of samples is finite, or else the probability of estimation error increases.

% \textbf{Complexity Control:}
Fig.~\ref{three_kl} top row compares the estimation of KL divergence with regular neural net and RKHS discriminator with complexity control based on eq. (\ref{aug_obj}). With our proposed RKHS discriminator, the KL estimates are significantly more reliable and accurate: error reduced from 0.5 to 0.04, 5.8 to 1.07 and 60.6 to 9.7 and variance reduced from 0.2 to 0.002, 223 to 4.4 and 3521 to 33 for true KL 1.3, 13.8 and 38.29 respectively. In Fig.~\ref{three_kl} bottom row, we investigate our complexity control method on the effect of varying the regularization parameter $\lambda = \lambda_0/m$. As expected, increasing regularization parameter penalizes more on the RKHS norm and therefore reduces variance. This is consistent with our theory. Regarding bias, however, as we increase the $\lambda$, the bias decreases and then starts to increase. Hence, one needs to strike a balance between bias and variance while choosing $\lambda$.
% \caption{The effect of the regularization parameter $\lambda$; hidden layer dimension = 20.}
% \label{beta-effect}
% \vspace{-2mm}
% \begin{center}
% % \setlength{\extrarowheight}{1}
% \begin{tabular}{ p{3em} p{1.7cm} p{1.7cm} p{1.7cm}}
% \hline \hline
% \multicolumn{1}{c}{\multirow{2}{4em}{\bf Lambda}}  &\multicolumn{3}{c}{\bf True KL} \\
% \cline{2-4}
% &1.3& 13.8 & 61.1\\
% \hline 

% 5e-5 & $1.46\pm 0.22$ & $16.65\pm 10.4$ & $116.7 \pm 116$\\
% \hline
% 1e-4 & $1.56\pm 0.25$ & $30.97\pm 10.5$ & $39.17 \pm 18.5$\\
% \hline
% 5e-4 & $1.47\pm 0.11$ & $13.44\pm 2.68$ & $18.36 \pm 3.9$\\
% \hline 
% \end{tabular}
% \end{center}
% \vspace{-3mm}
% \end{table}
% \vspace{-.2cm}

% \vspace{-.2cm}

\paragraph{Mutual Information Estimation}
Computation of mutual information is a direct use case of KL divergence computation. We replicate the experimental setup of \cite{poole2019variational, song2019understanding} to estimate mutual information between $(x,y)$ drawn from 20-d Gaussian distributions, where the mutual information is increased by step size of 2 from 2 to 10. We compare the performance of our method with traditional KL divergence computation methods like contrastive predictive coding (CPC) \cite{oord2018representation}, convex risk minimization (NWJ) \cite{nguyen2010estimating} and  SMILE \cite{song2019understanding}. 
In Fig.\ref{mi}, our method with RKHS discriminator (with $\lambda = 1e^{-5}$) performs better than CPC \cite{oord2018representation} and NWJ \cite{nguyen2010estimating}, and is competitive with the state-of-the-art, SMILE \cite{song2019understanding}.
In the bottom row, we also show the effect of regularization parameter $\lambda$ in our method. Similar to the previous experiment, increasing the regularization parameter decreases the variance and increases the bias. It is consistent with our theoretical insights about the effect of reducing RKHS norm on variance.
% Similarly, decreasing the Lipschitz constraint seems to increase the variance, although it is not very clear.
\begin{figure}[tbp]
\begin{center}
\centerline{\includegraphics[width=0.90\linewidth]{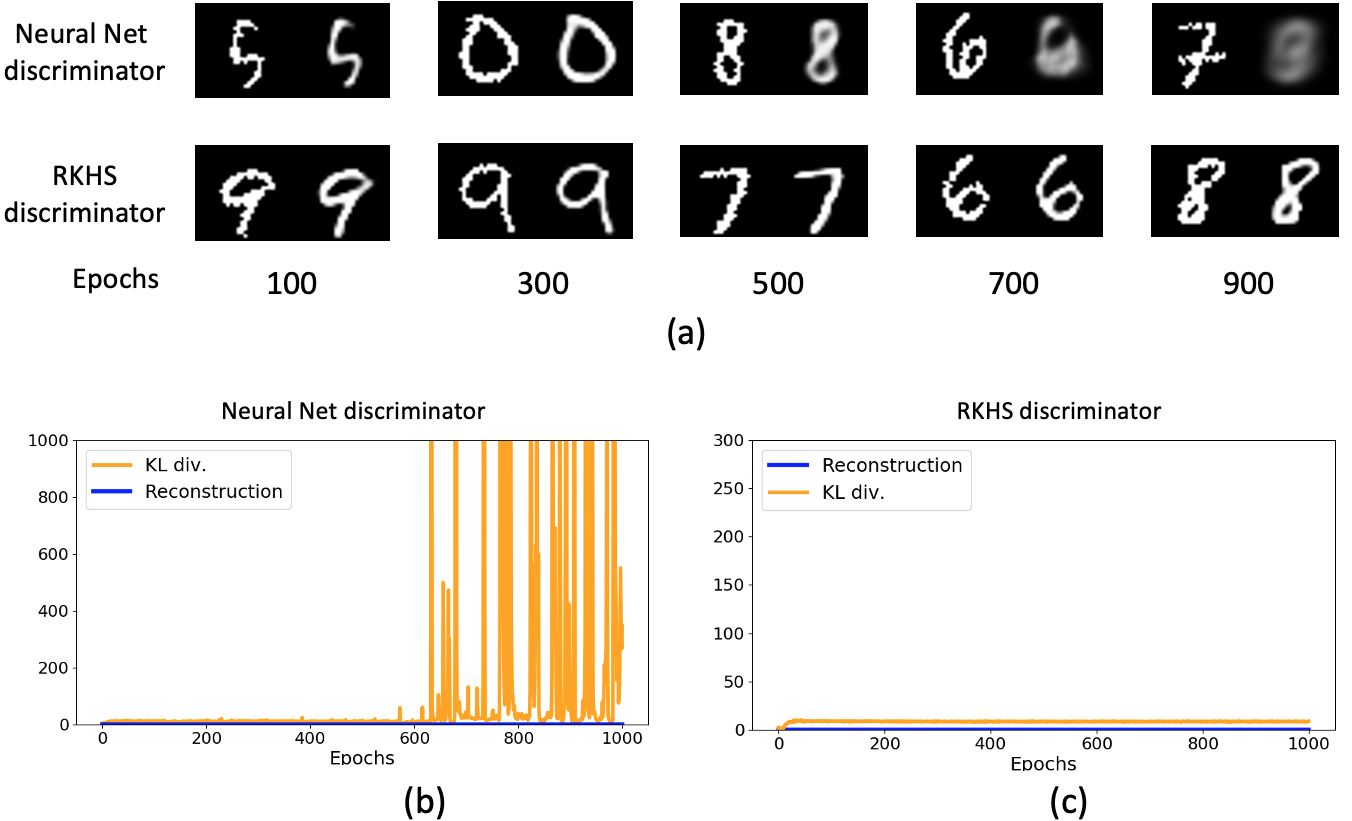}}
\caption{\small{(a) Comparison of MNIST digit reconstruction using AVB autoencoder model \cite{Mescheder2017ICML}. Trace of KL divergence and reconstruction loss in AVB model with Neural network discriminator (b) and RKHS discriminator in (c).}}
%i) infinite samples, ii) finite samples and a normal neural network discriminator, and iii) finite sample and the presented RKHS discriminator with complexity control.
% \vspace{-.2cm}}
\label{mnist_figures}
\end{center}
\vspace{-.8cm}
\end{figure}
\paragraph{Adversarial Variational Bayes}
Variational Bayes requires KL divergence estimation. When we do not have access to analytical form of the posterior/prior distributions, but only have access to the samples, we need to estimate KL divergence from samples. Adversarial Variational Bayes (AVB) \cite{Mescheder2017ICML} presents a way to achieve this using a discriminator network. We adopt this setup and demonstrate that the training becomes unstable if we do not constrain the complexity of the discriminator. 
First, we train AVB on MNIST dataset with a simple neural network discriminator architecture. As the training progresses, the KL divergence blows up after about 500 epochs (Fig. \ref{mnist_figures}(b)) and the reconstruction starts to get worse (Fig. \ref{mnist_figures}(a)). We modify the same architecture according to our construction such that it lies in RKHS and then penalize the RKHS norm as in eq. (\ref{aug_obj}). It stabilizes the training for a large number of epochs and the reconstruction does not deteriorate as the training progresses, resulting into sharp reconstruction (Fig. \ref{mnist_figures}(a)). We want to clarify that this instability in training neural net discriminator is present if we use a basic discriminator architecture. It does not mean that there exists no other method to design a stable neural net discriminator. In fact, AVB \cite{Mescheder2017ICML} presents a discriminator that adds additional inner product structure to stabilize the discriminator training. Our point here is that we can stabilize the training by ensuring that the discriminator lies in a well behaved function space (the RKHS) and controlling its complexity, consistent with our theory.

% \begin{figure}[tbp]
% \begin{center}
% \centerline{\includegraphics[width=\linewidth]{Figures/KL.png}}
% \caption{Tracing KL estimation over time reveals that the model collapses as the training progress further than 400 epochs }
% %i) infinite samples, ii) finite samples and a normal neural network discriminator, and iii) finite sample and the presented RKHS discriminator with complexity control.
% % \vspace{-.2cm}
% \label{kl_over_time}
% \end{center}
% % \vspace{-.5cm}
% \end{figure}
\section{Limitations, Discussion and Conclusion}
\paragraph{Limitations:} The proposed construction of neural function in RKHS exhibits good properties of both the deep learning and kernel methods. However, it requires constructing two separate deep networks, $\psi$ and $g$. It makes our model a bit bulky and also requires more parameter due to additional $g$. Moreover, currently our RKHS discriminator's output is scalar; generalizing this function to a multivariable output could make our model bulkier and increase parameters even more. 
Second limitation is the requirement of higher order derivative of kernel $K$ in assumption A3. While this requirement is satisfied if smooth activation function is used in $\phi_{\theta}$, for activations like ReLU or LeakyReLU, the derivatives exist everywhere except at the origin. In these cases, we need to carefully investigate if we can use subgradients to define operator norm $||\mathscr{L}_p||$.
% \section{Conclusion \& Discussion}
\paragraph{Discussion and Conclusion:}
We have shown that using a regular neural network as a discriminator in estimating KL divergence results in unreliable estimation if the complexity of the function space is not controlled.  We then showed %that this could be resolved 
a solution by constructing a discriminator function in RKHS space using neural networks and penalizing its complexity in a scalable way. Although the idea to use RKHS norm to penalize complexity is not new (see for example \cite{nguyen2010estimating}), it is not clear how to use this idea directly on the function $f$.
In traditional kernel methods, algorithms often do not work with RKHS function $f$ directly, but rather work with kernel matrix, $K$ by using, for example, the Representer Theorem \cite{scholkopf2002learning}. In the case of big data, working with the big kernel matrix is computationally expensive although some methods have been proposed to speed up the computation, like Random Fourier Feature \cite{rahimi2007random}. We propose a different view by directly constructing a function in RKHS space, which led us to scalable algorithm while incorporating the advantages of neural networks. Moreover, our representation could also be seen as an improvement over RFF by using neural basis, $\psi$, instead of Fourier basis.
%Adversarial training and discriminators are commonly used in GANs. 
The idea of constructing a {neural-net} function in RKHS and complexity control could also be useful in 
stabilizing GANs in general. % by controlling the complexity of the discriminator, 
% or potentially in improving generalization of neural networks. 
% Several papers have identified issues with the stability of GANs \cite{mescheder2018gan,kodali2017convergence,thanh-tung2018improving}. 
% One common understanding is that, in its raw form, we do not enforce the discriminator function to be smooth or regular around the neighborhood of its inputs. 
Currently, the most successful way to stabilize GANs is to enforce smoothness by gradient penalization \cite{arjovsky2017wasserstein,improved_wgan,binkowski2018demystifying}. On the light of the present analysis, gradient penalty could also be thought as a way to control the complexity of the discriminator. 

\bibliographystyle{ieeetr}
\bibliography{bibli}

\newpage
\appendix

\section{Problem Formulation and Contribution}
\paragraph{GAN-type Objective for KL Estimation}
Let $f$ be a discriminator, $f:\mathcal{X}\to {\rm I\!R}$. Let $p(x)$ and $q(x)$ be two probability density functions defined over the space $\mathcal{X}$. First, we train a discriminator as:
\begin{align}
\label{gan_form_kl}
    % E_{p(x)}\log \sigma(f(x))+E_{q(x)}\log (1-\sigma(f(x)))\\
    f^*=\underset{f}{\operatorname{argmax}}[{E_{p(x)}\log \sigma(f(x))+E_{q(x)}\log (1-\sigma(f(x)))}] %\vspace{-.2cm}
\end{align}
where $\sigma$ is the Sigmoid function given by $\sigma(x)=\frac{e^x}{1+e^x}$.
Then the KL divergence $KL(p(x)||q(x))$ is given by:
\begin{align}
\label{kl_expectation}
 KL(p(x)||q(x))=E_{p(x)}[f^*(x)]   
\end{align}
% \end{proposition*}
\begin{proof}
The proof is based on similar proofs in \cite{Mescheder2017ICML, sonderby2016amortised} and presented here for the sake of completeness.

We rewrite the objective as :
\begin{align}
    \int p(x) \log \sigma (f(x)) + q(x) \log (1- \sigma(f(x))) dx
\end{align}
This integral is maximum with respect to $f$ if and only if the integrand is maximal for every $x$. As argued in the Proposition 1 of \cite{goodfellow2014generative}, the function
\begin{align}
    t \mapsto a\log(t) + b \log(1-t)
\end{align}
attains its maximum at $t = \frac{a}{a+b}$ showing that,
\begin{align}
    \sigma(f^*(x)) = \frac{p(x)}{p(x)+q(x)}
\end{align}
Plugging the expression for Sigmoid function, we obtain,
\begin{align}
    f^*(x) = \frac{p(x)}{q(x)}
\end{align}
Therefore, by the definition of KL divergence, we have:
\begin{align}
    KL(p(x)||q(x))=E_{p(x)}[\frac{p(x)}{q(x)}] = E_{p(x)}[f^*(x)] 
\end{align}
\end{proof}

\section{Error Analysis and Control}
We start with the set of assumptions based on which our theory is developed.
\begin{enumerate}
    \item[A1.] The input domains $\mathcal{X}$ and $\mathcal{W}$ are compact. 
    \item[A2.] The functions $\phi_{\theta}$ and $g$ are Lipschitz continuous with Lipschitz constant $L_{\phi}$ and $L_g$ respectively.
    \item[A3.] Higher order derivatives $D_x^{\alpha}K(x,t)$ of kernel $K$ exist up to some high order $\tau =h/2$ .
\end{enumerate}
\vspace{0.5cm}

\begin{proposition}
Under the assumptions A1, A2, we have \\
i) $\underset{x,t}{sup}$ \hspace{0.1cm} $K_{\theta}(x,t) < \infty$, and \\
ii) $||g||^2_{\mathcal{L}_2(d\tau)} < \infty$.
\end{proposition}
\begin{proof}
i) By the definition $K_{\theta}(x, t) = \gamma \langle \phi_{\theta}(x), \phi_{\theta}(t) \rangle$. Using Cauchy Schwartz,
\begin{align}
    K_{\theta}(x, t) &\leq \gamma  ||\phi_{\theta}(x)|| ||\phi_{\theta}(t)|| \\
    &\leq  \gamma L_{\phi}||x|| L_{\phi}||t||\\
    &< \infty
\end{align}
where we used the fact that $\mathcal{X}$ is bounded, and therefore, $||x||$ and $||t||$ are finite.\\
ii) By definition,
\begin{align}
    ||g||^2_{\mathcal{L}_2(d\tau)} &= \int g(w)^2 d\tau(w)\\
    & \leq \int L_{g}^2 ||w||^2 d\tau(w)\\
    & = L_g^2 tr(C_w)
\end{align}
where $C_w$ is the uncentered covariance matrix of the Gaussian distributed w. Therefore, we immediately obtain $||g||^2_{\mathcal{L}_2(d\tau)} < \infty$.
\end{proof}

These results are useful in constructing a function $f$ in RKHS in Theorem 1 (Section 5) of the main paper.
\subsection{Bounding the Error Probability of KL Estimates}
We bound the deviation-from-mean error in two steps: 1) we derive a bound for a fixed kernel, 2) we take supremum of this bound over all the kernels parameterized by $\theta$. 

For a fixed kernel, we first bound the probability of deviation-from-mean error 
in terms of the covering number 
in Lemma \ref{sample_complexity_ap}. Then, we use an estimate of the covering number of RKHS due to \cite{cucker2002mathematical} to obtain a bound of error probability in terms of the kernel $K_{\theta}$ in Lemma \ref{complexity_ap}. Note that, Lemma \ref{complexity_ap} is proved for a fixed kernel $K_{\theta}$, where $\theta$ is fixed. Then finally in Theorem \ref{error_prob}, we take supremum over all kernels $K_{\theta}$s to obtain a bound on error probability on a space of functions with all possible kernels.

\begin{customlemma}{1}
\label{sample_complexity_ap}
Let $f^m_{\mathcal{H}_K}$ be the optimal discriminator function in a RKHS $\mathcal{H}_{K}$ which is M-bounded. Let ${KL}_m(f^m_{\mathcal{H}_K})=\frac{1}{m}\sum_i f^m_{\mathcal{H}_K}(x_i)$ and $KL(f^m_{\mathcal{H}_K}) = E_{ p(x)}[f^m_{\mathcal{H}_K}(x)]$ be the estimate of KL divergence from m samples and that by using true distribution $p(x)$ respectively.
% \begin{align}
%     \hat{KL}_m=\frac{1}{m}\sum_i f^*(x_i), and \hspace{0.2cm}
%     KL = E_{ p(x)}[f^*(x)]
% \end{align}
Then the probability of error at some accuracy level, $\epsilon$ is lower-bounded as:
\begin{align}
\nonumber
    \text{Prob.}(&|{KL}_m(f^m_{\mathcal{H}_K})-{KL}(f^m_{\mathcal{H}_K})|\leq \epsilon) 
    \geq 1-2\mathcal{N}(\mathcal{H}_K, \frac{\epsilon}{4\sqrt{S_K}})\exp(-\frac{m\epsilon^2}{4M^2})
\end{align}
where $\mathcal{N}(\mathcal{H}_K,\eta)$ denotes the covering number of a RKHS space $\mathcal{H}_K$ with disks of radius $\eta$, and $S_K=\underset{x,t}{sup} $\hspace{0.1cm} ${K(x,t)}$ which we refer as kernel complexity
\end{customlemma}
\begin{proof}
Let $\ell_z(f)=E_{p(x)}[f(x)]-\frac{1}{m}\sum_i f(x_i)$ denotes the error in the estimate such that we want to bound $|\ell_z(f)|$. We have,
\begin{align*}
% \nonumber
    &\ell_z(f_1)-\ell_z(f_2)
    %&=E_{p(x)}[f_1(x)]-\frac{1}{m}\sum_i f_1(x_i) -E_{p(x)}[f_2(x)]+\frac{1}{m}\sum_i f_2(x_i)\\
    = E_{p(x)}[f_1(x)- f_2(x)]-\frac{1}{m}\sum_i f_1(x_i) -  f_2(x_i)
\end{align*}
We know $E_{p(x)}[f_1(x)- f_2(x)]\leq ||f_1-f_2||_\infty$ and $\frac{1}{m}\sum_i f_1(x_i) -  f_2(x_i) \leq ||f_1-f_2||_\infty$.
Using the triangle inequality, we obtain
$
| \ell_z(f_1)-\ell_z(f_2) | \leq 2||f_1-f_2||_\infty
$. Now, consider $f\in \mathcal{H}_K$, then,
\begin{align}
    |f(x)|=|\langle K_x, f \rangle| \leq ||f||||K_x||=||f||\sqrt{K(x,x)}
\end{align}
This implies the RKHS space norm and $\ell_\infty$ norm of a function are related by 
\begin{align}
\label{sup_rkhs}
    ||f||_\infty \leq \sqrt{S_K}||f||_{\mathcal{H}_K}
\end{align}
Hence, we have: 
%Therefore, 
\begin{align}
\label{lipschitz}
    | \ell_z(f_1)-\ell_z(f_2) | \leq 2\sqrt{S_K}||f_1-f_2||_{\mathcal{H}_K}
\end{align}
The idea of the covering number is to cover the whole RKHS space $\mathcal{H}_K$ with disks of some fixed radius $\eta$, which helps us bound the error probability in terms of the number of such disks. Let $\mathcal{N}(\mathcal{H}_K,\eta)$ be such disks covering the whole RKHS space. Then, for any function $f$ in $\mathcal{H}_K$, we can find some disk, $D_j$ with centre $f_j$, such that $|| f-f_j ||_{\mathcal{H}_K} \leq \eta$. If we choose $\eta= \frac{\epsilon}{2\sqrt{S_K}}$, then from eq.(\ref{lipschitz}), we obtain,
\begin{align}
    \label{2eps}
    \underset{f\in D_j}{sup}{| \ell_z(f)| \geq 2\epsilon} \implies  | \ell_z(f_j)| \geq \epsilon
\end{align}
Using the Hoeffding's inequality,\hspace{0.1cm}
$
    \text{Prob.}(|\ell_z(f_j)|\geq \epsilon )\leq 2e^{-\frac{m\epsilon^2}{2M^2}}
$
and eq.(\ref{2eps}),
\begin{align}
    &\text{Prob.}(\underset{f\in D_j}{sup}{| \ell_z(f)| \geq 2\epsilon} )\leq 2e^{-\frac{m\epsilon^2}{2M^2}}
\end{align}

Applying union bound over all the disks, we obtian,
\begin{align}
    &\text{Prob.}(\underset{f\in \mathcal{H}}{sup}{| \ell_z(f)| \geq 2\epsilon} )\leq 2\mathcal{N}(\mathcal{H},\frac{\epsilon}{2\sqrt{S_K}})e^{-\frac{m\epsilon^2}{2M^2}}\\
    \nonumber
    &\text{Prob.}(\underset{f\in \mathcal{H}}{sup}{| \ell_z(f)| \leq \epsilon} )\geq 1- 2\mathcal{N}(\mathcal{H},\frac{\epsilon}{4\sqrt{S_K}})e^{-\frac{m\epsilon^2}{4M^2}}
\end{align}
which proves the lemma. \\

\underline{On M-boundedness of $f^m_{\mathcal{H}_K}$}\\
To prove the lemma, we assumed that $f^m_{\mathcal{H}_K}$ is M bounded. To see why this is reasonable, from eq.\ref{sup_rkhs}, we have $||f^m_{\mathcal{H}_K}||_\infty \leq \sqrt{S_K}||f^m_{\mathcal{H}_K}||_{\mathcal{H}_K} \leq \sqrt{S_K}||g||_{\mathcal{L}_2(d\rho)}$. Therefore, $f^m_{\mathcal{H}_K}$ is bounded if $S_K$ and $||g||_{\mathcal{L}_2(d\rho)}$ are bounded, which is true by Proposition 1.
\end{proof}
\begin{remark}
We derived the error bound based on the Hoeffding's inequality by assuming that our only knowledge about $f$ is that it is bounded. If we have other knowledge, for example, if we know the variance of $f$, we could use Bernstein's inequality instead of Hoeffding's inequality with minimal change to the proof. To the extent we are interested in the contribution of neural network in error bound, however, there is not much gain by using one inequality or the other. Hence, we stick with Hoeffding's inequality and note other possibilities.
\end{remark} 
\begin{remark}
Note that in Lemma 1, the radius of disks are inversely related to the the quantity, $S_K$, meaning that if $S_K$ is high, we would need large number of disks to fill the RKHS space. Hence, it denotes a quantity that reflects the complexity of the RKHS space. We, therefore, term it kernel complexity. Also in eq. \ref{sup_rkhs} and the discussion about the M-boundedness, we see that the maximum value $|f(x)|$ depends on $S_K$, again providing insight into how $S_K$ may control both maximum fluctuation and the boundedness.

\end{remark}
Lemma \ref{sample_complexity_ap} bounds the probability of error in terms of the covering number of the RKHS space. Next, we use Lemma 2 due to \cite{cucker2002mathematical} %gives the covering number of the RKHS space, which we use 
to obtain an error bound in estimating KL divergence with finite samples in Theorem \ref{error_prob}.

\begin{customlemma}{2}[\cite{cucker2002mathematical}]
\label{covering number_ap}
Let $K: \mathcal{X}\times \mathcal{X}\to {\mathbb{R}} $ is a $\mathcal{C}^\infty$ Mercer kernel and the inclusion $I_K:\mathcal{H}_K\xhookrightarrow{}\mathcal{C}(\mathcal{X})$ is the compact embedding defined by $K$ to the Banach space $\mathcal{C}(\mathcal{X})$ . Let $B_R$ be the ball of radius $R$ in RKHS $\mathcal{H}_{K}$. Then $\forall \eta>0, R >0, h>n $, we have
\begin{align}
    \ln \mathcal{N}(I_K(B_R), \eta) \leq \left( \frac{RC_h}{\eta} \right)^{\frac{2n}{h}}
\end{align}
where $\mathcal{N}$ gives the covering number of the space $I_K(B_R)$ with disks of radius $\eta$, and $n$ represents the dimension of inputs space $\mathcal{X}$. $C_h$ is given by 
\begin{align}
C_h=C\sqrt{||\mathscr{L}_K||}
\end{align}
where $L_K$ is a linear embedding from square integrable space $\mathcal{L}_2(d\rho)$ to the Sobolev space $H^{h/2}$ and $C$ is a constant.
\end{customlemma}
To prove Lemma \ref{covering number_ap}, the RKHS space is embedded in the Sobolev Space $H^{h/2}$ using $\mathscr{L}_K$ and then covering number of Sobolev space is used. Thus the norm of $\mathscr{L}_K$ and the degree of Sobolev space, $h/2$, appears in the covering number of a ball in $\mathcal{H}_K$. 
In Lemma \ref{complexity_ap}, we use this Lemma to bound the {estimation error of KL divergence}. 

\begin{customlemma}{3}
\label{complexity_ap}
Let ${KL}(f^m_{\mathcal{H}_{K_{\theta}}})$ and ${KL}_m(f^m_{\mathcal{H}_{K_{\theta}}})$ be the estimates of KL divergence obtained by using true distribution $p(x)$ and $m$ samples respectively and using a fixed kernel, ${K_{\theta}}$ as described in Lemma \ref{sample_complexity}, then the probability of error in the estimation at the error level $\epsilon$ is given by:
\begin{align*}
   \text{Prob.}(&|{KL}_m(f^m_{\mathcal{H}_{K_{\theta}}})-{KL}(f^m_{\mathcal{H}_{K_{\theta}}})|\geq \epsilon) \leq 2\exp\Bigg[\left( \frac{4RC_s\sqrt{S_K(\theta)||\mathscr{L}_s(\theta)||}}{\epsilon} \right)^{\frac{2n}{h}}-\frac{m\epsilon^2}{4M^2}\Bigg]
\end{align*}
\end{customlemma}
\begin{proof}
Lemma \ref{covering number_ap} gives the covering number of a ball of radius $R$ in an RKHS space. In Lemma \ref{sample_complexity_ap}, if we consider the hypothesis space to be a ball of radius $R$, we can apply Lemma \ref{covering number} in it. Additionally, since we fix the radius of disks to be $\eta=\frac{\epsilon}{4\sqrt{S_K}}$ in Lemma \ref{sample_complexity_ap}, we obtain,
\begin{align}
\nonumber
    &\text{Prob.}(|{KL}_m(f^m_{\mathcal{H}_{K_{\theta}}})-{KL}(f^m_{\mathcal{H}_{K_{\theta}}})|\geq \epsilon) \leq 2\exp\Big[\left( \frac{4\sqrt{S_{K_{\theta}}}RC_h}{\epsilon} \right)^{\frac{2n}{h}}-\frac{m\epsilon^2}{4M^2}\Big]\\
%   \therefore &\text{Prob.}(|{KL}_m(f^m_{\mathcal{H}_K})-{KL}(f^m_{\mathcal{H}_K})|\leq \epsilon) 
%     \geq 1-2\exp\Bigg[\left( \frac{4RC_s\sqrt{S_K||L_s||}}{\epsilon} \right)^{\frac{2n}{h}}-\frac{m\epsilon^2}{4M^2}\Bigg]
\end{align}
Substituting $C_h=C\sqrt{||\mathscr{L}_{K_{\theta}}||}$, we obtain,
\begin{align}
% \nonumber
%     &\text{Prob.}(|{KL}_m(f^m_{\mathcal{H}_K})-{KL}(f^m_{\mathcal{H}_K})|\leq \epsilon) \geq 1-2\exp\Big[\left( \frac{4\sqrt{S_K}RC_h}{\epsilon} \right)^{\frac{2n}{h}}-\frac{m\epsilon^2}{4M^2}\Big]\\
 &\text{Prob.}(|{KL}_m(f^m_{\mathcal{H}_K})-{KL}(f^m_{\mathcal{H}_{K_{\theta}}})|\geq \epsilon) 
    \leq 1-2\exp\Bigg[\left( \frac{4RC\sqrt{S_{K_{\theta}}||\mathscr{L}_{K_{\theta}}||}}{\epsilon} \right)^{\frac{2n}{h}}-\frac{m\epsilon^2}{4M^2}\Bigg]
\end{align}
\end{proof}
\begin{customthm}{2}
\label{error_prob}
Let ${KL}(f^m_{\mathcal{H}})$ and ${KL}_m(f^m_{\mathcal{H}})$ be the estimates of KL divergence obtained by using true distribution $p(x)$ and $m$ samples respectively as described in Lemma \ref{sample_complexity_ap}, then the probability of error in the estimation at the error level $\epsilon$ is given by:
\begin{align*}
   \text{Prob.}(&|{KL}_m(f^m_{\mathcal{H}})-{KL}(f^m_{\mathcal{H}})|\leq \epsilon) \geq 1-2\exp\Bigg[\left( \frac{4RC_p\sqrt{S_p||\mathscr{L}_p||}}{\epsilon} \right)^{\frac{2n}{h}}-\frac{m\epsilon^2}{4M^2}\Bigg]
\end{align*}
where $C_p\sqrt{S_p||\mathscr{L}_p||} = \underset{K_{\theta}}{sup} \hspace{0.2cm} C_s\sqrt{S_K(\theta)||\mathscr{L}_s||}$, i.e. $C_p, S_p, \mathscr{L}_p$ correspond to a kernel for which the bound is maximum.
\end{customthm}
\begin{proof}
Lemma \ref{complexity_ap} gives an error bound for a fixed kernel, $K_{\theta}$. To find an upper bound over all possible kernels, we take the supremum over all kernels.
\begin{align}
\label{sup_K_ap}
   \text{Prob.}(|{KL}_m(f^m_{\mathcal{H}})-{KL}(f^m_{\mathcal{H}})|\geq \epsilon) &\leq \underset{K_{\theta}}{sup} \hspace{0.2cm} \text{Prob.}(|{KL}_m(f^m_{\mathcal{H}_{K_{\theta}}})-{KL}(f^m_{\mathcal{H}_{K_{\theta}}})|\geq \epsilon)\\
   \label{theta_p_ap}
   & \leq 2\exp\Bigg[\left( \frac{4RC_p\sqrt{S_p||\mathscr{L}_p||}}{\epsilon} \right)^{\frac{2n}{h}}-\frac{m\epsilon^2}{4M^2}\Bigg]
\end{align}
where $S_p = S_K(\theta_p)$ and $\mathscr{L}_p = \mathscr{L}_K(\theta_p)$, \textit{i.e.,} $S_p$ and $\mathscr{L}_p$ correspond to kernel complexity and Sobolev operator norm corresponding to optimal kernel $K_{\theta_p}$ that extremizes eq. (\ref{sup_K_ap}). Theorem statement readily follows from eq. (\ref{theta_p_ap})
\end{proof}

\section{Variance and Consistency of the Estimate}
\subsection{Variance Analysis}
\begin{customthm}{3}
\label{variance_ap}
Let $X = KL_m(f^m_{\mathcal{H}})$ be the estimated KL divergence using m samples as described in Theorem \ref{error_prob}. Assuming that $X$ follows a Gaussian distribution $X \sim \mathcal{N}(\mu, \sigma)$, we can obtain an upper bound on this variance of the estimate as follows:
\begin{align}
    \sigma &\leq \frac{\epsilon}{\sqrt{2}\text{erf}^{-1}\Big[-4\exp\Big[\left( \frac{4RC_p\sqrt{S_p||\mathscr{L}_p||}}{\epsilon} \right)^{\frac{2n}{h}}-\frac{m\epsilon^2}{4M^2}\Big] + 1 \Big]}
\end{align}
where erf is the Gauss error function 
\begin{align}
    \text{erf}(x) = \frac{2}{\sqrt{\pi}}\int_0^x e^{-t^2}dt
\end{align}
and it is a monotonic function.
\end{customthm}
\begin{proof}
$X$ follows a Gaussian distribution with mean $\mu$ and variance $\sigma$. Let its cumulative distribution function be $\Phi_{\mu, \sigma}$. By definition,
\begin{align}
    P(X \leq \hat{x}) = \Phi_{\mu, \sigma}(\hat{x})\\
    P(X \geq \hat{x}) = 1 - \Phi_{\mu, \sigma}(\hat{x})\\
    \label{cdf}
    P(X -\mu \geq \epsilon) = 1 - \Phi_{\mu, \sigma}(\mu + \epsilon)
\end{align}
Since two sided probability is higher than one sided, we have,
\begin{align}
    P(X -\mu \geq \epsilon)  &\leq P(|X -\mu| \geq \epsilon) \\
    &\leq 2\exp\Bigg[\left( \frac{4RC_s\sqrt{S_K||\mathscr{L}_s||}}{\epsilon} \right)^{\frac{2n}{h}}-\frac{m\epsilon^2}{4M^2}\Bigg]
\end{align}
where we used Theorem 2.
Using eq.\ref{cdf}, we have,
\begin{align}
    \label{main_exp}
    1 - \Phi_{\mu, \sigma}(\mu + \epsilon) \leq 2\exp\Bigg[\left( \frac{4RC_s\sqrt{S_K||\mathscr{L}_s||}}{\epsilon} \right)^{\frac{2n}{h}}-\frac{m\epsilon^2}{4M^2}\Bigg]
\end{align}
For a Gaussian distribution, we can use the following expression for the cumulative distribution function,
\begin{align}
    \Phi_{\mu, \sigma}(\hat{x}) = \frac{1}{2}\Big[1 + \text{erf} \big(\frac{\hat{x}-\mu}{\sigma \sqrt{2}}\big)\Big]
\end{align}
where erf is the Gauss error function.
Using this in the eq.\ref{main_exp}, 
\begin{align}
    1 - \text{erf} \big(\frac{\epsilon} {\sigma \sqrt{2}}\big) \leq 4\exp\Bigg[\left( \frac{4RC_s\sqrt{S_K||\mathscr{L}_s||}}{\epsilon} \right)^{\frac{2n}{h}}-\frac{m\epsilon^2}{4M^2}\Bigg]\\
    \text{erf} \big(\frac{\epsilon} {\sigma \sqrt{2}}\big) \geq -4\exp\Bigg[\left( \frac{4RC_s\sqrt{S_K||\mathscr{L}_s||}}{\epsilon} \right)^{\frac{2n}{h}}-\frac{m\epsilon^2}{4M^2}\Bigg] + 1 
\end{align}
Since the function erf is invertible within domain (-1,1), we have,
\begin{align}
    \frac{\epsilon} {\sigma \sqrt{2}} &\geq \text{erf}^{-1}\Bigg[-4\exp\Big[\left( \frac{4RC_s\sqrt{S_K||\mathscr{L}_s||}}{\epsilon} \right)^{\frac{2n}{h}}-\frac{m\epsilon^2}{4M^2}\Big] + 1 \Bigg] \\
    \sigma &\leq \frac{\epsilon}{\sqrt{2}\text{erf}^{-1}\Big[-4\exp\Big[\left( \frac{4RC_s\sqrt{S_K||\mathscr{L}_s||}}{\epsilon} \right)^{\frac{2n}{h}}-\frac{m\epsilon^2}{4M^2}\Big] + 1 \Big]}
\end{align}
\end{proof}

\subsection{Consistency of Estimates}

\begin{customthm}{4}
Let $f^*$ and $f^m_h$ and $f^*_h$ be optimal discriminators defined as
\begin{align}
    f^* &=\underset{f}{\operatorname{argmax}}[{E_{p(x)}\log \sigma(f(x))+E_{q(x)}\log (1-\sigma(f(x)))}]\\
    \label{f_star}
    f^*_h &=\underset{f \in h}{\operatorname{argmax}}[{E_{p(x)}\log \sigma(f(x))+E_{q(x)}\log (1-\sigma(f(x)))}] \\
    \label{fm}
    f^m_h &=\underset{f \in h}{\operatorname{argmax}}\Big[{\frac{1}{m}\sum_{x_i \sim p(x_i)}\log \sigma(f(x_i))+\frac{1}{m}\sum_{x_j \sim q(x_j)}\log (1-\sigma(f(x_j)))}\Big]- \frac{\lambda_0}{m}||g||_{\mathcal{L}_2(d\tau)}^2
\end{align}
and the  KL estimate is given by $KL(f)=E_{p(x)}[f(x)],   \hspace{0.4cm}  KL_m(f)=\frac{1}{m}\sum_{x_i \sim p(x_i)}[f(x)]$. Then, in the limiting case as $m \to \infty$, $|KL_m(f^m_h)-KL(f^*)| \to 0$.
\end{customthm}
\begin{proof}
Estimation error can be divided into three terms as 
\begin{align}
\label{total_error_ap}
KL_m(f^m_h)-KL(f^*)=\underbrace{KL_m(f^m_h)-KL(f^m_h)}_{\text{Deviation-from-mean error}}+\underbrace{KL(f^m_h)-KL(f^*_h)}_{\text{Discriminator induced error}}+\underbrace{KL(f^*_h)-KL(f^*)}_{Bias}
\end{align}
Therefore,
\begin{align}
\small
\nonumber
|KL_m(f^m_h)-KL(f^*)|&\leq |{KL_m(f^m_h)-KL(f^m_h)}|+|{KL(f^m_h)-KL(f^*_h)}|\\
&+|{KL(f^*_h)-KL(f^*)}|
\end{align}
To show that the total error goes to zero, we show that each term on the right goes to zero. The last term is the bias and we assume that the RKHS space $h = \mathcal{H}$ we consider consists the true solution, $f^*$. Hence the bias goes to zero.

Using Theorem \ref{error_prob}, it is immediately clear that the first term, $|{KL_m(f^m_h)-KL(f^m_h)}|$ approaches zero in the limiting case as $m \to \infty$.

The only remaining is the second term, $|{KL(f^m_h)-KL(f^*_h)}|$. In 
Theorem \ref{disc_error} we show that this term also goes to zero as $m \to 0$. 
\end{proof}

% Using these assumptions, we prove the following theorem which plays key role in proving consistency of the proposed estimator.

\begin{customthm}{5}
\label{disc_error}
Let $f^*_h$ and $f^*_h$ be the optimal discriminators as defined in eq. (\ref{f_star}) and eq. (\ref{fm}), and the KL divergence estimate using discriminators learned using finite and infinite samples be $KL(f^m_h)=\int[f^m_h(x)]p(x)dx$ and $KL(f^*_h)=\int[f^*_h(x)]p(x)dx$, where,
% \begin{align*}
% f^*_h&=\underset{f \in h}{\text{argmax}}[{E_{p(x)}\log \sigma(f(x))+E_{q(x)}\log (1-\sigma(f(x)))}]\\
% f^m_h&=\underset{f \in h}{\text{argmax}}\Big[{\frac{1}{m}\sum_{x_i \sim p(x_i)}\log \sigma(f(x_i))
%     +\frac{1}{m}\sum_{x_j \sim q(x_j)}\log (1-\sigma(f(x_j)))}\Big] - \frac{\lambda_0}{m}||g||_{\mathcal{L}_2(d\tau)}^2
% \end{align*}
Then, in the limiting case, we have
$$\underset{m\to \infty}{lim} |KL(f^m_h)-KL(f^*_h)| = 0 $$
\end{customthm}
\begin{proof}
\begin{align*}
   |KL(f^m_h)-KL(f^*_h)|&=|\int[f^m_h(x)-f^*_h(x)]p(x)dx |\\
   &\leq \underset{x}{sup}{|f^m_h(x)-f^*_h(x)|}=||f^m_h(x)-f^*_h(x)||_{\infty}
\end{align*}
Therefore, we can show $\underset{m\to \infty}{lim} KL(f^m_h)-KL(f^*_h) = 0 $ if $\underset{m\to \infty}{lim} ||f^m_h(x)-f^*_h(x)||_{\infty} = 0$, that is, if the function $f^m_h(x)$ converges uniformly to function $f^*_h(x)$ in the limiting case.

The two maximizer functions are given by 
\begin{align}
f^*_h&=\underset{f \in h}{\text{argmax}}[{E_{p(x)}\log \sigma(f(x))+E_{q(x)}\log (1-\sigma(f(x)))}]\\
f^m_h&=\underset{f \in h}{\text{argmax}}\Big[{\frac{1}{m}\sum_{x_i \sim p(x_i)}\log \sigma(f(x_i))
    +\frac{1}{m}\sum_{x_j \sim q(x_j)}\log (1-\sigma(f(x_j)))}\Big] - \frac{\lambda_0}{m}||g||^2
\end{align}
As a first step in showing that $f^m_h$ uniformly approaches $f^*_h$, we first show that $\underset{m\to \infty}{lim} \frac{\lambda_0}{m}||g||^2 = 0$ in Lemma \ref{regularization_lemma}. 

Then, to prove the rest, let us denote,
\begin{align*}
    G_m(f)&={\frac{1}{m}\sum_{x_i \sim p(x_i)}\log \sigma(f(x_i))
    +\frac{1}{m}\sum_{x_j \sim q(x_j)}\log (1-\sigma(f(x_j)))}\\
    G(f)&={E_{p(x)}\log \sigma(f(x))+E_{q(x)}\log (1-\sigma(f(x)))}
\end{align*}

In Lemma \ref{concavity_lemma}, we prove that functionals $G(f)$ and $G_m(f)$ are concave with respect to function $f$. In the light of these two lemmas, we argue 
\begin{align}
    \underset{m\to \infty}{lim} ||f^m_h(x)-f^*_h(x)||_{\infty} = 0 \hspace{0.2cm}\text{if} \hspace{0.2cm} \underset{m\to \infty}{lim} \underset{f}{\sup}|G_m(f)-G(f)| = 0
\end{align}
Next, we show $\underset{m\to \infty}{lim} \underset{f}{\sup}|G_m(f)-G(f)| = 0$ as follows.
We have,
\begin{align}
\nonumber
    |G_m-G|=&|{\frac{1}{m}\sum_{x_i \sim p(x_i)}\log \sigma(f(x_i))
    +\frac{1}{m}\sum_{x_j \sim q(x_j)}\log (1-\sigma(f(x_j)))}\\
    &- {E_{p(x)}\log \sigma(f(x))+E_{q(x)}\log (1-\sigma(f(x)))}|\\
    \nonumber
    \leq& |\frac{1}{m}\sum_{x_i \sim p(x_i)}\log \sigma(f(x_i))- E_{p(x)}\log \sigma(f(x))|\\
    &+|\frac{1}{m}\sum_{x_j \sim q(x_j)}\log (1-\sigma(f(x_j)))- E_{q(x)}\log (1-\sigma(f(x)))|\\
\nonumber
    \therefore \underset{m\to \infty}{lim} \underset{f}{\sup}&|G_m(f)-G(f)|
    \leq \underset{m\to \infty}{lim} \underset{f}{\sup}|\frac{1}{m}\sum_{x_i \sim p(x_i)}\log \sigma(f(x_i))- E_{p(x)}\log \sigma(f(x))|\\
    &\hspace{1cm}+\underset{m\to \infty}{lim} \underset{f}{\sup}|\frac{1}{m}\sum_{x_j \sim q(x_j)}\log (1-\sigma(f(x_j)))- E_{q(x)}\log (1-\sigma(f(x)))|
\end{align}
Both the terms on right hand side go to zero if $\log \circ \sigma \circ f$ is in a Glivenko Cantelli class of functions using Empirical Process Theory \cite{van1996weak}, which we prove in Lemma \ref{glivenko_cantelli}. That completes the proof.
\end{proof}

\begin{customlemma}{4}
\label{regularization_lemma}
$\underset{m\to \infty}{lim} \frac{\lambda_0}{m}||g||^2 = 0$
\begin{proof}
$||g||_{\mathcal{L}_2(d\rho)}$ is bounded because $g$ is Lipschitz continuous and its domain is bounded. Since, $||g||_{\mathcal{L}_2(d\rho)}$ is bounded, we immediately obtain the required statement.
\end{proof}
\end{customlemma}
\begin{customlemma}{5}
\label{concavity_lemma}
The functional $G(f)$ is concave with respect to function $f$ in the following sense:
$\theta_1G(f_1)+\theta_2G(f_2) \leq G(\theta_1f_1+\theta_2f_2)$ for any $\theta_1,\theta_2 \in (0,1)$ such that $\theta_1+\theta_2=1$. The same is true for $G_m(f).$
\end{customlemma}
\begin{proof}
\begin{align}
\nonumber
    \theta_1G(f_1)+\theta_2G(f_2)&=\theta_1 \Big[\int p(x)\log \sigma(f_1(x))dx+\int q(x)\log (1-\sigma(f_1(x)))dx \Big]\\
    &+\theta_2 \Big[\int p(x)\log \sigma(f_2(x))dx+\int q(x)\log (1-\sigma(f_2(x)))dx \Big]\\
    \nonumber
    &=\int p(x)\Big[\theta_1\log \sigma(f_1(x))dx + \theta_2\log \sigma(f_2(x))dx \Big]\\
    \label{neg_f}
    &+\int q(x)\Big[\theta_1\log \sigma(-f_1(x))dx + \theta_2\log \sigma(-f_2(x))dx \Big]\\
    \label{ineq}
    \nonumber
    &\leq \int p(x)\log \sigma[\theta_1f_1(x) + \theta_2f_2(x)]dx\\
    &+\int q(x)\log \sigma[-(\theta_1f_1(x) + \theta_2f_2(x))]dx\\
    &=G(\theta_1 f_1+ \theta_2 f_2)
\end{align}
where we used the fact that $\log (1-\sigma(f(x)))=\log \sigma(-f(x))$ (this is straightforward using definition of Sigmoid function, $\sigma$) in line \ref{neg_f}. In line \ref{ineq}, we used the fact that $\log \sigma$ is a concave function (see Lemma \ref{log_sigma}).
\end{proof}
\begin{customlemma}{6}
\label{glivenko_cantelli}
$\log \circ \sigma \circ f$ is a Glivenko Cantelli class of function.
\begin{proof}
In Lemma \ref{lip_f}, we show that, by definition, $f$ is Lipschitz continuous with some Lipschitz constant $L_f$. In Lemma \ref{log_sigma} we show that if $f$ is a Lipschitz continuous function from $\mathcal{X}$ to $(-\infty, \infty)$ with Lipschitz constant, $L_f$, then $\log\sigma f$ is a function from $\mathcal{X}$ to $(-\infty,0)$ with same Lipschitz constant $L_f$. Hence, $v = \log \sigma f$ is a a function from $\mathcal{X}$ to $(-r, 0)$. Note that since $\mathcal{X}$ is bounded and $f$ is Lipschitz continuous from $\mathcal{X}$ to $\mathbb{R}$, we can always find some $r$ such that $v$ maps from $\mathcal{X}$ to $(-r, 0)$.

Now, we show that $v = \log \sigma f$ is Glivenko Cantelli by entropy number. Let $\mathcal{V} =\{v : v = \log(\sigma(f)), f \in \mathcal{F}\}$. In Lemma \ref{gc}, we use theorem from \cite{van1996weak} to show that $\mathcal{V}$ is Glivenko Cantelli if and only if
\begin{align}
    \frac{1}{m}\log N(\epsilon, \mathcal{V}_M, \ell_1(\mathbb{P}_m)) \overset{\mathbb{P}}{\to} 0,
\end{align}
for any $M>0, \epsilon$, where $\mathcal{V}_M$ is the class of functions $v\textbf{1}\{E \leq M\}$ where $v$ ranges over $\mathcal{V}$ and $E$ is an envelope function to $\mathcal{V}$.
Since we proved that $\log(\sigma(f)(x)) < 0$ for any $x$, we can choose $E = v_0(x) = \textbf{0}$ as a constant function that is an envelope to $\mathcal{V}$. For any $M>0$, therefore, $1\{E \leq M\} = 1$ trivially and $\mathcal{V}_M =\mathcal{V}$. Hence, we just need to show 
\begin{align}
   \frac{1}{m}\log N(\epsilon, \mathcal{V}, \ell_1(\mathbb{P}_m)) \overset{\mathbb{P}}{\to} 0
\end{align}
In Lemma \ref{entropy_n}, we show that the entropy number of such a function is given by
\begin{align}
    \log \mathcal{N}(\epsilon,\mathcal{V}, \ell_1(\mathbb{P}_m) ) \leq  \left( \frac{16L. diam(\mathcal{X})}{\epsilon} \right)^{ddim(\mathcal{X})} \log \left(\frac{4r}{\epsilon}\right)
\end{align}
and therefore is bounded and independent of the sample size $m$. Hence, $\frac{1}{m}\log N(\epsilon, \mathcal{V}, \ell_1(\mathbb{P}_m))$ goes to $0$.
\end{proof}
\end{customlemma}
\begin{customlemma}{7}
\label{lip_f}
The function $f$ defined in Theorem 1 on the main paper as: 
\begin{align}
     f(x)=\int_{\mathcal{W}}g(w)\psi(x,w)d\tau(w),
\end{align}
where $\psi(x,w)=\phi_{\theta}(x)^Tw$ and the function $\phi_{\theta}$ is Lipschitz continuous with Lipschitz constant $L_{\phi}$. Then, the function $f$ is Lipschitz continuous with some Lipschitz constant, $L_f$.
\end{customlemma}
\begin{proof}
By the definition,
\begin{align}
    f(x) = \langle g(w), \psi(x,w) \rangle_{\mathcal{L}_2(d\tau)}
\end{align}
For any two points $x_1$ and $x_2$, 
\begin{align}
    |f(x_1) - f(x_2)| &= \langle g(w), \psi(x_1,w) - \psi(x_2,w) \rangle_{\mathcal{L}_2(d\tau)}\\
    \label{inner_tau}
    & \leq ||g(w)||_{\mathcal{L}_2(d\tau)} ||\psi(x_1, w) - \psi(x_2, w)||_{\mathcal{L}_2(d\tau)}
\end{align}
where we used Cauchy Schwartz. Now, taking the difference in $\psi$, it can be written as
\begin{align}
    ||\psi(x_1, w) - \psi(x_2, w)||_{\mathcal{L}_2(d\tau)} &= \sqrt{\int [\psi(x_1, w) - \psi(x_2, w)]^2 d\tau(w)}\\
    &= \sqrt{\int [(\phi_{\theta}(x_1) - \phi_{\theta}(x_2))^T w]^2 d\tau(w)}\\
    \label{inner_rd}
    & \leq \sqrt{\int ||\phi_{\theta}(x_1) - \phi_{\theta}(x_2)||^2 ||w||^2 d\tau(w)}
\end{align}
where we again used Cauchy Schwartz in the last line since $[(\phi_{\theta}(x_1) - \phi_{\theta}(x_2))^T w]$ is an inner product in $\mathbb{R}^D$ where D is the dimension of $w$. Since $\phi_{\theta}$ is Lipschitz continuous with Lipschitz constant $L_{\phi}$, we have 
$$
||\phi_{\theta}(x_1) - \phi_{\theta}(x_2)|| \leq L_{\phi} ||x_1 - x_2||
$$
Using this inequality in eq.\ref{inner_rd}, we obtain
\begin{align}
      ||\psi(x_1, w) - \psi(x_2, w)||_{\mathcal{L}_2(d\tau)} &\leq L_{\phi} ||x_1 - x_2|| \sqrt{\int ||w||^2 d\tau(w)}\\
      \label{psi_ineq}
      &= L_{\phi} ||x_1 - x_2|| \sqrt{tr(C_w)}
\end{align}
where, $C_w$ is the uncentered covariance matrix of Gaussian distributed $w$. Plugging eq.(\ref{psi_ineq}) in eq.(\ref{inner_tau}), we obtain
\begin{align}
    |f(x_1) - f(x_2)| \leq ||g(w)||_{\mathcal{L}_2(d\tau)}L_{\phi} \sqrt{tr(C_w)} ||x_1 - x_2|| 
\end{align}
Since, we have that $||g(w)||_{\mathcal{L}_2(d\tau)} < \infty$ (see Lemma \ref{regularization_lemma}), we have proved that $f$ is Lipschitz continuous with Lipschitz constant given by $L_f \leq ||g(w)||_{\mathcal{L}_2(d\tau)}L_{\phi} \sqrt{tr(C_w)}$.
\end{proof}

\begin{customlemma}{8}
\label{log_sigma}
The function $\log \circ \sigma$ exhibits following properties:\\
i) It is  a concave function with its derivative always between $0$ and $1$ \\
ii) If the Lipschitz constant of $f$ is $L_f$, so is the Lipschitz constant of $\log \circ \sigma \circ f$
\begin{proof}
i) Let us denote $u(x) = \log (\sigma(x))$. Then, we have,
\begin{align}
    &u(x) = \log \frac{e^x}{1+e^x} = x - \log (1+e^x)\\
    \therefore &u'(x) = 1 - \frac{e^x}{1+e^x}  = \frac{1}{1+e^x}\\
    \therefore &0 <  u'(x) < 1,  \hspace{0.5cm} \forall x \in (- \infty, \infty)
\end{align}
which proves that the derivative is between $0$ and $1$. To show that $u(x)$ is concave, it is sufficient to note that its second derivative is always negative.

ii) Let us use notation $u = \log (\sigma)$, and let $f_2 = f(x_2)$, $f_1 = f(x_1)$, $u_2 = u(f(x_2))$, $u_1 = u(f(x_1))$. Since the maximum derivative of $u$ is upper bounded by 1, $u$ as a function of $f$ has Lipschitz constant $1$ and therefore, we can write
\begin{align}
  u_2 - u_1 = u(f_2) - u(f_1) \leq f_2 - f_1 &= f(x_2) -f(x_1)  \\
  &  \leq L_f||x_2 -x_1||
\end{align}
where the last inequality is because $f$ is Lipschitz continuous with Lipschitz constant $L_f$. This proves that the Lipschitz constant of $\log \circ \sigma \circ f$ is also $L_f$.

\end{proof}
\end{customlemma}
\begin{customlemma}{9}
\label{entropy_n}
Let $\mathcal{F}_L$ be the space of L-Lipschitz functions mapping the metric space ($\mathcal{X}, \rho$) to [0,r]. Let $ddim(\mathcal{X})$ and $diam(\mathcal{X})$ denote the doubling dimension and diameter of $\mathcal{X}$ respectively. Then,\\
i) the covering numbers of $\mathcal{F}_L$  can be estimated in terms of the covering numbers of $\mathcal{X}$:
\begin{align}
    \mathcal{N}(\epsilon,\mathcal{F}_L, ||.||_{\infty} ) \leq \left(\frac{4r}{\epsilon}\right)^{\mathcal{N}(\epsilon/{8L},\mathcal{X}, ||.||_{\infty} )}
\end{align}
ii) the entropy number of $\mathcal{F}_L$  can be estimated as:
\begin{align}
    \log \mathcal{N}(\epsilon,\mathcal{F}_L, ||.||_{\infty} ) \leq  \left( \frac{16L. diam(\mathcal{X})}{\epsilon} \right)^{ddim(\mathcal{X})} \log \left(\frac{4r}{\epsilon}\right)
\end{align}
iii) the entropy number with respect to $\ell_1(\mathbb{P}_m) = \int |f|d\mathbb{P}_m = \frac{1}{m} \sum_k |f(x_k)|$ defined with respect to the $m$ input points, is the same as (ii), \textit{i.e.}
\begin{align}
    \log \mathcal{N}(\epsilon,\mathcal{F}_L, \ell_1(\mathbb{P}_m) ) \leq  \left( \frac{16L. diam(\mathcal{X})}{\epsilon} \right)^{ddim(\mathcal{X})} \log \left(\frac{4r}{\epsilon}\right)
\end{align}
where $\mathbb{P}_m$ is an empirical probability measure with respect to $m$ inputs points in $\mathcal{X}$.

\begin{proof}
The proof is adapted from \cite{Kontorovich2014} Lemma 2 and \cite{gottlieb} Lemma 6, and modified to handle range $[0,r]$.

i) We first cover the domain $\mathcal{X}$ by $N$ balls ${U_1, U_2, ..., U_{|N|}}$, where $N = \mathcal{N}(\epsilon/{8L},\mathcal{X}, ||.||_{\infty} )$ is the covering number of $\mathcal{X}$, $N = \{x_i \in U_i\}^{|N|}_{i=1}$ is a set of center points of $|N|$ balls and $\epsilon' = \epsilon/{8L}$ is the radius of the covering balls.

Now, our strategy is to construct an $\epsilon$ cover $\hat{F}=\{ \hat{f_1},..., \hat{f}_{| \hat{F}|} \}$ for $\mathcal{F}_L$ with respect to $||.||_{\infty}$. To do so, at every point $x_i \in N$, we choose the value of $\hat{f}(x_i)$ to be some multiple of $2L\epsilon' = \frac{\epsilon}{4}$, while maintaining $||\hat{f}||_{Lip} \leq 2L$. We then construct a 2L-Lipschitz extension for $\hat{f}$ from $N$ to all over $\mathcal{X}$(note that such an extension always exists, see \cite{mcshane,whitney_extension}). 

With this construction, we can show that every $f \in \mathcal{F}_L$ is close to some $\hat{f} \in \hat{F}$ in the sense that $||f - \hat{f}||_{\infty} \leq \epsilon$. To show this, note the following:
\begin{align}
    |f(x) - \hat{f}(x)| & \leq |f(x) - f(x_N)| + |f(x_N) - \hat{f}(x_N)| + |\hat{f}(x_N) - \hat{f}(x)|\\
    \label{lips}
    & \leq L.\rho(x, x_N)+\epsilon / 4 +2L. \rho(x, x_N) \\
    & \leq \epsilon
\end{align}
where the inequality in eq.\ref{lips} is due to the fact that $f$ is $L$-Lipschitz and $\hat{f}$ is $2L$-Lipschitz and since we have covered the input space $\mathcal{X}$, each $x$ is within $\epsilon'$ of some $x_N$. Also note that for every $f(x_N)$ we can find $\hat{f}(x_N)$ within some radius $\epsilon /4$; this is because we choose $f(x_N)$ to be some multiple of $2L\epsilon'$. Finally, we need to compute the cardinality of $\hat{F}$, i.e. $|\hat{F}|$. For any $x_i \in |N|$, $\hat{f}$ can take one of the multiple of $2L\epsilon'$ values. Hence, there are $r/2L\epsilon'$ such possibilities as the range is $[0,r]$. Since there are $|N|$ such possibilities for $x_i$, the upper bound on all possible function values $\hat{f}$ is $(\frac{r}{2L\epsilon'})^{|N|} = (\frac{4r}{\epsilon})^{|N|}$, which proves the first statement after plugging in the value of $|N|$.

ii) Taking logarithm of the result in i)
\begin{align}
\label{entropy_x}
    \log \mathcal{N}(\epsilon,\mathcal{F}_L, ||.||_{\infty} ) \leq {\mathcal{N}(\epsilon/{8L},\mathcal{X}, ||.||_{\infty} )} \log \left(\frac{4r}{\epsilon}\right)
\end{align}

The covering number of the input space, $\mathcal{X}$ in terms of doubling dimension, $ddim(\mathcal{X})$ and diameter, $diam(\mathcal{X})$ can be written as \cite{krauthgamer2004navigating}:
\begin{align}
    {\mathcal{N}(\epsilon,\mathcal{X}, ||.||_{\infty} )}  \leq \left( \frac{2 diam(\mathcal{X})}{\epsilon} \right)^{ddim(\mathcal{X})}
\end{align}
Plugging this expression in eq.(\ref{entropy_x}), we obtain the required expression.

iii) The result in i) is with respect to $||.||_{\infty}$. In eq.(\ref{lips}), we showed that for any $f \in \mathcal{F}_L$ there is some $\hat{f} \in \hat{F}$ within a radius of $\epsilon$ such that $||f - \hat{f}||_{\infty} \leq \epsilon$. Here, we show that this also implies that $||f - \hat{f}||_{\ell_1(\mathbb{P}_m)} \leq \epsilon$. We show this as follows:
\begin{align}
    ||f - \hat{f}||_{\ell_1(\mathbb{P}_m)} &= \frac{1}{m} \sum_{k=1}^m |f(x_k)-\hat{f}(x_k)| \\
    &\leq \frac{1}{m} \sum_{k=1}^m \epsilon = \epsilon
\end{align}
Therefore, the entropy number with respect to ${\ell_1(\mathbb{P}_m)}$ metric is same as the entropy number with respect to the $||.||_{\infty}$, which proves our third claim.
\end{proof}
\end{customlemma}

\begin{customlemma}{10}[\cite{van1996weak} Theorem 3.5. ]
\label{gc}
Let $\mathcal{V}$ be a class of measurable functions with envelope $E$ such that $P(E) < \infty$. Let $\mathcal{V}_M$ be the class of functions $v.\textbf{1}\{E \leq M\}$ where $v$ ranges over $\mathcal{V}$. Then, $\mathcal{V}$ is a Glivenco Cantelli class of functions, \textit{i.e.} it satisfies 
\begin{align}
    \underset{v \in \mathcal{V}}{sup} \hspace{0.1cm}|\mathbb{P}_m v -Pv|
\end{align}
, if and only if
\begin{align}
    \frac{1}{m}\log N(\epsilon, \mathcal{V}_M, L_1(\mathbb{P}_m)) \overset{\mathbb{P}}{\to} 0,
\end{align}
for every $\epsilon >0$ and $M>0$, where $Pv = \int v dP$ and $\mathbb{P}_m v = \frac{1}{m} \sum_k v(x_k)$.
\end{customlemma}

\section{Experimental Results}
\paragraph{Code:} The code will be publicly released.

\subsection{Two Gaussian}
\subsubsection{Architecture and Implementation}
RKHS Discriminator Architecture (Pytorch Code)
\begin{lstlisting}
class RKHS_Net(nn.Module):
    def __init__(self, dim =10, mid_dim1=20, mid_dim2=20, mid_dim3=20, D=50, gamma =1, metric = 'rbf', lip=5, g_lip =5):
        super(RKHS_Net, self).__init__()
        self.gamma = torch.FloatTensor([gamma])
        self.metric = metric
        self.D = D
        self.act = nn.ReLU()
        self.lin1 = spectral_norm( nn.Linear(dim, mid_dim1), k =g_lip)
        self.lin2 = spectral_norm( nn.Linear(mid_dim1, mid_dim2), k =g_lip)
        self.lin3 = spectral_norm( nn.Linear(mid_dim2 , mid_dim3), k =g_lip)
        self.lin4 = spectral_norm( nn.Linear(mid_dim3, 1), k =g_lip)

        self.g = nn.Sequential(self.lin1,
                               self.act,
                               self.lin2,
                               self.act,
                               self.lin3,
                               self.act,
                               self.lin4
                               )

        self.lin_phi1 = spectral_norm(nn.Linear(2, mid_dim1), k=lip)
        self.lin_phi2 = spectral_norm(nn.Linear(mid_dim1, mid_dim2), k=lip)
        self.lin_phi3 = spectral_norm(nn.Linear(mid_dim2, mid_dim3), k=lip)
        self.lin_phi4 = spectral_norm(nn.Linear(mid_dim3, dim), k=lip)

        self.phi = nn.Sequential(self.lin_phi1,
                               self.act,
                               self.lin_phi2,
                               self.act,
                               self.lin_phi3,
                               self.act,
                               self.lin_phi4
                               )

    def forward(self, y):
        x=self.phi(y)
        d = x.shape[1]
        if self.metric =='rbf':
            w= torch.sqrt(2*self.gamma)*torch.randn(size=(self.D,d))
        w=w.to(x.device)
        psi = ((torch.matmul(x,w.permute(1,0)) ))*(torch.sqrt(2/torch.FloatTensor([self.D])).to(x.device))
        w_a = w
        g= self.g(w_a)
        f = (psi*g.permute(1,0)).mean(1)
        g_norm =(g**2).mean()
        return f, g_norm
\end{lstlisting}
Simple Neural Network Discriminator Architecture ( Pytorch Code)
\begin{lstlisting}
class DNet_basic(nn.Module):
    def __init__(self, input_dim, mid_dim1, mid_dim2, output_dim, lip_constraint = False, lip = 5):
        super(DNet_basic, self).__init__()

        self.act = nn.ReLU()
        self.lin1 = nn.Linear(input_dim, mid_dim1)
        self.lin2 = nn.Linear(mid_dim1, mid_dim2)
        self.lin3 = nn.Linear(mid_dim2, mid_dim2)
        self.lin4 = nn.Linear(mid_dim2, output_dim)
        # self.sigmoid=nn.Sigmoid()

        
        self.phi = nn.Sequential(self.lin1,
                                 self.act,
                                 self.lin2,
                                 self.act,
                                 self.lin3,
                                 self.act,
                                 self.lin4
                                 )

    def forward(self, x):
        t = self.phi(x)
        return t
\end{lstlisting}

\paragraph{Discrete approximation:} Both the discriminators have stacked Fully connected layers and activation function. In the proposed RKHS discriminator, we have an additional network \texttt{self.g} which we use to approximate the continuous integral $f(x)=\int_{\mathcal{W}}g(w)\psi(x,w)d\tau(w)$ with the following discrete approximation:
\begin{align}
\label{rkhs_discrete}
    f(x) =\frac{1}{D} \sum_{k=1}^D g(w_k)(\phi^Tw_k \sqrt{\frac{2}{D}})
\end{align}
where $w$ is sampled from a Normal distribution with variance $\gamma$. In our experiments $D = 500$ was sufficient. Note that the Neural network discriminator is similar to $\phi ^T w$, except that $w$ is not randomly sampled and there is no $g$ network.

\paragraph{Lipschitz constraints:}: To enforce Lipschitz constraints on network $g$ and $\phi$ consistent with our assumptions and theoretical results, we use spectral normalization in the RKHS discriminator while it is absent in the basic Neural network discriminator.

\subsubsection{Data and Hyperparameters}
\textbf{Data}: Since this is a toy experiment, data were generated locally using pytorch command \texttt{randn} to sample from Gaussian distribution.

\textbf{Learning rate:} $5\times 10^{-3} $ (both models)\\
\textbf{No. of samples from each distribution:} $2500$ (both models)\\
\textbf{Minibatch size:} 50 (both models)\\
$\boldsymbol{\lambda:} 0.005$ (RKHS disc.)\\
\textbf{Hyperparameter selection:} (RKHS disc.) The hyperparameters like learning rate and $\lambda$
 were selected by first estimating KL divergence at a mid value like $13$. Then, same value was used in all experiments.
\subsubsection{Computational Resources and Time}
Running one experiment of KL divergence calculation takes 74 s for the basic algorithm while it takes 245 s for the proposed method in a single GeForce GTX 1080 Ti GPU with 11GB memory.

% \begin{figure}[t]
% % \vskip 0.2in
% \begin{center}
% \centerline{\includegraphics[width=\linewidth]{Figures/beta-kl.png}}
% \caption{The effect of the regularization parameter $\lambda$ in KL estimates (y-axis) plotted against the varying hidden layer dimension for each KL divergence value.}
% \label{betakl}
% \end{center}
% \vspace{-6mm}
% \end{figure}

\subsection{Mutual Information Estimation}

\subsubsection{Models, Architecture and Implementation}
RKHS Discriminator Architecture (Pytorch Code)
\begin{lstlisting}
class ConcatLipFeatures(nn.Module):
    def __init__(self, dim, hidden_dim, layers, activation,lip, gamma =1, metric = 'rbf', D=500, mid_dim=5, g_lip =2, **extra_kwargs):
        super(ConcatLipFeatures, self).__init__()
        self.gamma = torch.FloatTensor([gamma])
        self.metric = metric
        self.D = D
        self.act = nn.ReLU()
        self.lin1 = spectral_norm( nn.Linear(hidden_dim, mid_dim), k = g_lip)
        self.lin2 = spectral_norm( nn.Linear(mid_dim, mid_dim), k = g_lip)
        self.lin3 = spectral_norm( nn.Linear(mid_dim, mid_dim), k = g_lip)
        self.lin4 = spectral_norm( nn.Linear(mid_dim, 1), k = g_lip)

        self.g = nn.Sequential(self.lin1,
                               self.act,
                               self.lin2,
                               self.act,
                               self.lin3,
                               self.act,
                               self.lin4
                               )
        # output of this layer is d dim features
        self.rkhs_layer = feature_perceptron(dim * 2, hidden_dim, 1, layers, activation, lip)

    def forward(self, x, y):
        batch_size = x.size(0)
        # Tile all possible combinations of x and y
        x_tiled = torch.stack([x] * batch_size, dim=0)
        y_tiled = torch.stack([y] * batch_size, dim=1)
        # xy is [batch_size * batch_size, x_dim + y_dim]
        xy_pairs = torch.reshape(torch.cat((x_tiled, y_tiled), dim=2), [
                                 batch_size * batch_size, -1])
        # Compute features for each x_i, y_j pair.
        phi = self.rkhs_layer(xy_pairs)
        d = phi.shape[1]
        if self.metric == 'rbf':
            w = torch.sqrt(2 * self.gamma) * torch.randn(size=(self.D, d))
        w = w.to(x.device)
        psi = ((torch.matmul(phi, w.permute(1, 0)))) * (torch.sqrt(2 / torch.FloatTensor([self.D])).to(x.device))
        w_a = w  # torch.cat((w,u.permute(1,0)),1)
        g = self.g(w_a)
        f = (psi * g.permute(1, 0)).mean(1)
        g_norm = (g ** 2).mean()
        return f, g_norm
\end{lstlisting}
Simple Neural Network Discriminator Architecture (Pytorch Code)
\begin{lstlisting}
class ConcatCritic(nn.Module):
    def __init__(self, dim, hidden_dim, layers, activation, **extra_kwargs):
        super(ConcatCritic, self).__init__()
        # output is scalar score
        self._f = mlp(dim * 2, hidden_dim, 1, layers, activation)

    def forward(self, x, y):
        batch_size = x.size(0)
        # Tile all possible combinations of x and y
        x_tiled = torch.stack([x] * batch_size, dim=0)
        y_tiled = torch.stack([y] * batch_size, dim=1)
        # xy is [batch_size * batch_size, x_dim + y_dim]
        xy_pairs = torch.reshape(torch.cat((x_tiled, y_tiled), dim=2), [
                                 batch_size * batch_size, -1])
        # Compute scores for each x_i, y_j pair.
        scores = self._f(xy_pairs)
        return torch.reshape(scores, [batch_size, batch_size]).t()
\end{lstlisting}

Similar to the previous experiment, the RKHS discriminator and the Neural network discriminator are similar in core design. The main difference lies in that the RKHS discriminator has this inner product construction same as eq.(\ref{rkhs_construct}) in previous subsection. To achieve this construction, the RKHS discriminator an additional network, $\texttt{self.g}$ and enforces Lipschitz constraint through spectral normalization, which are absent in simple Neural network discriminator.
\subsubsection{Data and Hyperparameters}
\textbf{Data:} The experimental setup and data generation follow \texttt{https://github.com/ermongroup/smile-mi-estimator}.\\
\underline{\textbf{Common for all methods}}\\
\textbf{batch size:} 64\\
\textbf{no. of layers:} 2\\
\textbf{hidden dim:} 256\\
\textbf{no. of iterations:} 40000\\
\textbf{learning rate:} $5 \times 10^{-4}$

\underline{\textbf{Specific to the proposed method}}\\
$\gamma:$5\\
Lipschitz constant enforced, $L_{\phi}$ (layer wise): 5\\
Lipschitz constant enforced, $L_{g}$ (layer wise): 5

\subsubsection{Computational Resources and Time}
GPU: GeForce RTX 2080 Ti 11 GB

Below, we report time taken by each method to complete an experiment to obtain mutual information between two 20-d Gaussian distributed random variables using $40,000$ samples from each distribution and mutual information increasing stepwise.
\begin{table}[h]
  \caption{Time taken to complete one experiment}
%   \label{sample-table}
  \centering
  \begin{tabular}{llll}
    \toprule
    % \multicolumn{2}{c}{Part}                   \\
    % \cmidrule(r){1-2}
    CPC     & NWJ     & SMILE & Ours (RKHS disc.) \\
    \midrule
    52 s & 48 s  & 52 s & 63 s    \\
    \bottomrule
  \end{tabular}
\end{table}

\subsubsection{Existing Assets}
We used the code from the repo \texttt{https://github.com/ermongroup/smile-mi-estimator} to generate data as well as run baseline mutual information methods. This code corresponds to the Song et al. \cite{song2019understanding}.

\subsection{Adversarial Variational Bayes}

\subsubsection{Models, Architecture and Implementation}
RKHS Discriminator Architecture (Pytorch Code)
\begin{lstlisting}
class Discriminator_RKHS(nn.Module):
    def __init__(self, x_dim, h_dim, z_dim, lip = 5, g_lip = 5, dim = 10, mid_dim1 = 20, mid_dim2 = 20, mid_dim3 = 20, D=100, gamma =1, metric = 'rbf'):
        super(Discriminator_RKHS, self).__init__()
        self.metric = metric
        self.gamma = torch.FloatTensor([gamma])
        self.D = D
        self.phi = nn.Sequential(
            spectral_norm(nn.Linear(x_dim + z_dim, h_dim), k = lip),
            nn.LeakyReLU(),
            spectral_norm(nn.Linear(h_dim, h_dim), k = lip),
            nn.LeakyReLU(),

            spectral_norm(nn.Linear(h_dim, h_dim), k = lip),
            nn.LeakyReLU(),
            spectral_norm(nn.Linear(h_dim, h_dim), k = lip),
            nn.LeakyReLU(),
            spectral_norm(nn.Linear(h_dim, int(h_dim/4)), k = lip)

        )

        self.act = nn.ReLU()
        self.lin1 = spectral_norm(nn.Linear(int(h_dim/4), mid_dim1), k=g_lip)
        self.lin2 = spectral_norm(nn.Linear(mid_dim1, mid_dim2), k=g_lip)
        self.lin3 = spectral_norm(nn.Linear(mid_dim2, mid_dim3), k=g_lip)
        self.lin4 = spectral_norm(nn.Linear(mid_dim3, 1), k=g_lip)

        self.g = nn.Sequential(self.lin1,
                               self.act,
                               self.lin2,
                               self.act,
                               self.lin3,
                               self.act,
                               self.lin4
                               )

    def weight_init(self, mean, std):
        for m in self._modules:
            normal_init(self._modules[m], mean, std)

    def forward(self, y, z):
        y = y.view(y.shape[0], -1)
        y = torch.cat([y, z], 1)
        x =self.phi(y)
        d = x.shape[1]

        if self.metric == 'rbf':
            w = torch.sqrt(2 * self.gamma) * torch.randn(size=(self.D, d))
        w = w.to(x.device)
        psi = ((torch.matmul(x, w.permute(1, 0)))) * (torch.sqrt(2 / torch.FloatTensor([self.D])).to(x.device))
        w_a = w  
        g = self.g(w_a)
        f = (psi * g.permute(1, 0)).mean(1)
        g_norm = (g ** 2).mean()
        return f, g_norm
\end{lstlisting}

Simple Neural Network Discriminator Architecture (Pytorch Code)
\begin{lstlisting}
class Discriminator_simple(nn.Module):
    def __init__(self, x_dim, h_dim, z_dim):
        super(Discriminator_simple, self).__init__()
        self.net = nn.Sequential(
            nn.Linear(x_dim + z_dim, h_dim),
            nn.LeakyReLU(),
            nn.Linear(h_dim, h_dim),
            nn.LeakyReLU(),

            nn.Linear(h_dim, h_dim),
            nn.LeakyReLU(),
            nn.Linear(h_dim, h_dim),
            nn.LeakyReLU(),
            nn.Linear(h_dim, int(h_dim/4)),
            nn.LeakyReLU(),
            nn.Linear(int(h_dim/4), 1)
        )

    def weight_init(self, mean, std):
        for m in self._modules:
            normal_init(self._modules[m], mean, std)

    def forward(self, x, z):
        x = x.view(x.shape[0], -1)
        x = torch.cat([x, z], 1)
        out =self.net(x)
        # x = x + torch.sum(z ** 2, 1)
        return out
\end{lstlisting}
\subsubsection{Data and Hyperparameters}
\textbf{Data:} Standard MNIST dataset is used. 

\textbf{Learning rate:} $ 10^{-3} $ (both models)\\
\textbf{Minibatch size:} 1024 (both models)\\
\textbf{Hidden dim of encoder/decoder:} 800 (both)\\
\textbf{Hidden dim discriminator:} 1024 (both)\\
$\boldsymbol{\lambda:} 1$ (RKHS disc.)\\

\subsubsection{Computational Resources and Time}
GPU: GeForce GTX 1080 Ti 11GB\\
Time taken to train MNIST for 1000 epochs using AVB with simple Neural net discriminator: 11.3 hrs\\
Time taken to train MNIST for 1000 epochs using AVB with RKHS discriminator: 14.7 hrs

\subsubsection{Existing Assets}
We followed the official implementation of Adversarial Variational Bayes \cite{Mescheder2017ICML} at \texttt{https://github.com/LMescheder/AdversarialVariationalBayes}

\section{Societal Impacts}
% This paper focuses on a theoretical problem of estimating KL divergence using a discriminator function. The main contribution of the paper is to understand why such a procedure might lead to unreliable and unstable estimates and how to prevent that using complexity control of the discriminator function space. 
We discuss possible negative impacts in two categories: 1) Impact of theoretical contribution, 2) Impact of applications

\paragraph{Societal Impact of theoretical contribution:} The main theoretical contribution of the paper is its connection between reliable/stable estimation and complexity analysis of the discriminator function space. In its general form, this contribution does not, by itself, pose any negative societal impact. Rather, it is about stabilizing algorithms. So, it contributes towards more robust and stable algorithms, and may help in developing more secure applications. We do not foresee any negative societal impacts in safety and security of human beings and automatic systems, human rights, human livelihood or economic security, environment. We do not see it causing theft, harassment, fraud, bias or discrimination.

\paragraph{Societal Impact of possible applications:} As demonstrated in the experiment section, this work can be applied to information theoretic applications that require mutual information or KL divergence estimation. For example, it has been used in generative modeling like variational autoencoder, variational Bayes or in stabilizing generative adversarial networks (GANs). These generative modeling techniques are, by themselves, quite general and can have numerous applications, including the ones with negative impacts. By helping in accurate estimation of KL divergence and by providing theoretical analysis, this work is contributing to develop stronger generative models and by extension could be indirectly helping in their negative uses. In that aspect, we appeal everyone using the algorithms and ideas in this paper to be thoughtful and responsible in their use.
\end{document}